\documentclass[nonatbib]{article}
\usepackage[preprint]{neurips_2025}
\usepackage[numbers,sort&compress]{natbib}
\usepackage[dvipsnames]{xcolor}
\usepackage{graphicx}
\usepackage{subcaption}

\usepackage{mathtools}      
\usepackage{amssymb}        
\usepackage{amsthm}
\newtheorem{theorem}{Theorem}

\usepackage{algorithm}
\usepackage{algpseudocode}   
\usepackage[none]{hyphenat}

\usepackage{booktabs}
\usepackage{multirow}
\usepackage{makecell}
\usepackage{multicol}
\usepackage{enumitem}

\usepackage[font=small,labelfont=bf,tableposition=top]{caption}
\DeclareCaptionLabelFormat{andtable}{#1~#2 \& \tablename~\thetable}

\usepackage[most]{tcolorbox}
\tcbuselibrary{breakable,skins}

\usepackage{titlesec}
\usepackage{appendix}

\usepackage{hyperref}

\DeclareMathOperator*{\argmin}{arg\,min}

\title{Mix- and MoE-DPO: A Variational Inference Approach to Direct Preference Optimization}

\author{
  Jason Bohne\thanks{ Equal contribution.} \\
  Department of Applied Mathematics and Statistics\\
  Stony Brook University\\
  Stony Brook, NY 11794\\
  \texttt{jason.bohne@stonybrook.edu}
  \And
  Pawe\l{} Polak\footnotemark[1] \\
  Department of Applied Mathematics and Statistics\\
  Institute for Advanced Computational Science \\
  Center of Excellence in Wireless and Information Technology (CEWIT)\\
  AI Innovation Institute\\
  Stony Brook University\\
  Stony Brook, NY 11794\\
  \texttt{pawel.polak@stonybrook.edu}
  \And
  David Rosenberg\\
  Bloomberg\\
  Toronto, ON M5J 2S1\\
  \texttt{drosenberg44@bloomberg.net}
  \And
  Brian Bloniarz\\
  Bloomberg\\
  San Francisco, CA 94105\\
  \texttt{bbloniarz@bloomberg.net}
  \And
  Gary Kazantsev\\
  Bloomberg\\
  New York, NY 10022\\
  \texttt{gkazantsev@bloomberg.net}
}

\newtheorem{corollary}{Corollary}[theorem]
\newtheorem{lemma}[theorem]{Lemma}

\newtheoremstyle{assumptionstyle} 
  {3pt}
  {3pt} 
  {\itshape} 
  {} 
  {\bfseries} 
  {.} 
  {.5em} 
  {} 

\theoremstyle{assumptionstyle}

\begin{document}
\maketitle

\begin{abstract}
Direct Preference Optimization (DPO) has recently emerged as a simple and effective alternative to reinforcement learning from human feedback (RLHF) for aligning large language models (LLMs) with user preferences. However, existing DPO formulations rely on a single monolithic model, which limits their expressivity in multi-task settings and their adaptability to heterogeneous or diverse preference distributions. In this work, we propose Mix- and MoE-DPO, a framework that extends DPO with both soft mixture models and mixture-of-experts (MoE) architectures, using a stochastic variational inference approach. Our method introduces a latent-variable model over expert assignments and optimizes a variational evidence lower bound (ELBO), enabling stable and efficient learning of specialized expert policies from preference data. Mix- and MoE-DPO provides three key advantages over standard DPO: (i) generalization via universal function approximation through mixtures; (ii) reward and policy specialization through expert components tailored to distinct preference modes; and (iii) contextual alignment through input-dependent soft gating that enables user-specific mixture policies. Our framework supports both shared base architectures with expert-specific policy heads and fully independent expert models, allowing flexible trade-offs between parameter efficiency and specialization. We validate our approach on a variety of model sizes and multi-preference datasets, demonstrating that Mix- and MoE-DPO offers a powerful and scalable method for preference-based LLM alignment.
\end{abstract}

\section{Introduction}

Aligning large language models (LLMs) with human preferences is a central objective in building safe and reliable AI systems. Direct Preference Optimization (DPO) \cite{rafailov2023direct} has emerged as a widely adopted and computationally efficient alternative to Reinforcement Learning from Human Feedback (RLHF) \cite{christiano2017deep,ouyang2022training,lee2023rlaif} for the alignment of LLMs. Using a direct optimization framework over preference pairs, DPO avoids the need for explicit reward modeling and optimization complexities required in RLHF methods, while still achieving competitive performance with traditional methods.

Despite its advantages, standard DPO methods are still inherently limited by their reliance on a single, monolithic policy. In scenarios with multi-expert or heterogeneous preferences arising from varied user groups, task domains, or annotation styles, this uniformity restricts expressivity and may induce suboptimal alignment. Although extensions of DPO \cite{azar2024general,wu2024,morimura2024filtered,gu2025mask,shi2024direct,lai2024step,zeng2024token,liu2024provably} offer improvements, such modifications to date have been limited to a single-policy framework.

To address the limitations of single-policy preference optimization, we propose a modular framework for aligning large language models (LLMs) with heterogeneous human preferences. While Direct Preference Optimization (DPO) has demonstrated strong empirical performance, its monolithic structure limits its capacity to represent diverse user behaviors, task-specific feedback, or multi-objective alignment. This limitation is particularly significant in light of the growing ecosystem of domain-specialized LLMs, which offer reusable components for structured preference modeling.

We introduce a mixture-based generalization of DPO by modeling the policy as a latent mixture over expert components. Each expert specializes in a distinct preference mode, and the overall model is trained via a variational inference procedure grounded in a Mixture-of-Bradley--Terry (MBT) likelihood. This framework establishes a component-wise policy–reward equivalence and accommodates two architectural regimes: (i) expert-specific heads on a shared encoder for parameter-efficient sharing, and (ii) independently parameterized expert models for maximal specialization. In both cases, expert policies can be initialized from task-specific pretrained heads or independent models for fine-tuning or deployment in a zero-shot setting, enabling efficient and scalable adaptation.

To optimize the mixture model, we propose a stochastic variational EM algorithm that alternates between inferring expert responsibilities and updating policies and rewards. We instantiate this in two variants: \emph{Mix-DPO}, with fixed mixture weights updated via posterior averaging, and \emph{MoE-DPO}, which uses a soft gating network to assign input-dependent or user-specific weights. Expert policies are trained via closed-form updates minimizing expert-specific KL-regularized objectives, while the gating network is optimized to match predicted weights to inferred posteriors.

Our framework is designed for compatibility with modular deployment in real-world systems. Expert components can be swapped, added, or reused with minimal retraining. In shared-encoder settings, new capabilities can be added efficiently via head specialization. Moreover, MoE-DPO supports user personalization by conditioning the gating network on user metadata, allowing for user-specific mixtures without modifying the expert policies. This flexibility makes the framework well suited for scalable alignment in heterogeneous and multi-task environments.\\

\textbf{Contributions.} This work makes the following theoretical and empirical contributions:\\
\textbf{Multi-Expert Preference Alignment.} We generalize DPO to a latent mixture of expert policies, combining a Mixture-of-Bradley--Terry model with KL-regularized reward-based objectives.\\
\textbf{Variational Training Algorithm.} We develop a scalable variational EM algorithm for modular policy optimization under latent expert assignments, preserving closed-form policy updates.\\
\textbf{Multi-Reward Generalization with Mix-DPO.} We show that Mix-DPO improves generalization across multiple reward signals by enabling posterior specialization of expert policies.\\
\textbf{Contextual Multi-Task Alignment with MoE-DPO.} We demonstrate that MoE-DPO enables effective user- or task-specific routing via input-dependent gating, yielding improved multi-domain alignment.\\
\textbf{Modular Deployment and Personalization.} We highlight that our framework supports integration of pretrained models and can enable efficient customization without retraining via gating.

The paper is structured as follows. Section \ref{sec:model} provides our mixture-model formulation for the policy, reward, and preference model. Section \ref{sec:algorithm} outlines our variational training framework and special cases of Mix-DPO and MoE-DPO. Section \ref{sec:experiments} presents experimental results for preference alignment of language models on a variety of tasks. Thorough related work on preference alignment methods and variational inference, along with proofs and additional experimental results, are in the Appendix.\\

\textbf{Notation.} We use the following notation throughout. Prompt–response–preference triplets are denoted as \( \{(x_i, y_i^+, y_i^-)\}_{i=1}^n \sim \mathcal{D} \), where \( y_i^+ \succ y_i^- \) indicates the preferred response for prompt \( x_i \sim p \). The base policy is denoted \( \pi(y \mid x) \), the reward function is \( r(x, y) \), and the temperature parameter is \( \beta \). The latent expert index is denoted \( z \in \{1, \ldots, K\} \), with conditional prior \( p(z = k \mid x) = w_k(x) \), where \( w_k(x) \) is the mixture weight assigned to expert \( k \). Accordingly, the expert-specific policy is \( \pi_k(y \mid x) \), the expert reward is \( r_k(x, y) \), and the reference policy is \( \pi_{\mathrm{ref}(k)}(y \mid x) \). The variational posterior over expert assignments for a preference triplet is denoted \( q_k(x, y^+, y^-) \).

\section{Multi-Expert Preference Alignment}\label{sec:model}

Let \( \{(x_i, y_i^+, y_i^-)\}_{i=1}^n \) denote a dataset of preference triplets, where each prompt \( x_i \sim p \) is associated with two responses sampled from a policy \( \pi \), and \( y_i^+ \) is preferred over \( y_i^- \). To support modular specialization and better capture heterogeneous user preferences or task distributions, we extend DPO to a mixture-of-experts formulation with input-dependent gating:
\begin{equation}\label{eq:mixpolicy}
\pi(y \mid x) := \sum_{k=1}^K w_k(x) \pi_k(y \mid x),
\end{equation}
where each \( \pi_k \) is an expert policy and \( w_k(x) \geq 0 \) are gating weights satisfying \( \sum_{k=1}^K w_k(x) = 1 \). The gating function assigns context-dependent responsibilities to expert components, enabling the model to adaptively route different inputs to specialized policies. This structure increases expressivity while preserving compositional modularity and aligns with recent empirical evidence that combining specialized models enhances generalization in multi-domain LLM settings~\cite{wang2024interpretable,xu2024perfect}.

To define the reward associated with the policy in~\eqref{eq:mixpolicy}, we aggregate the expert-specific rewards \( r_k(x, y) \) using a softmax-style combination:
\begin{equation}\label{eq:mix_exp_reward}
r(x, y) := \log \left( \sum_{k=1}^K w_k(x) \exp(r_k(x, y)) \right).
\end{equation}
This reward corresponds to a smooth maximum over the expert rewards and recovers the single-expert case when the gating weights are deterministic; i.e.,~\eqref{eq:mix_exp_reward} reduces to \( r(x,y)=\sum_{k=1}^K w_k(x) r_k(x,y) \) when \( w_k(x) \in \{0,1\} \).

We assume the number of mixture components in the policy and reward spaces is equal, with shared gating function \( w_k(x) \) across \( k = 1, \ldots, K \). This modeling choice reflects the assumption that the same latent structure governs both generation and evaluation: the specific type of expert responsible for generating a response is also responsible for assessing its quality by assigning a reward to the output. This alignment ensures consistency between policy optimization and reward modeling and is particularly realistic in settings where expert components reflect stable sources of heterogeneity. Examples of such heterogeneity include user types, task domains, or annotation protocols, as considered in our experiments.

Analogous to the original DPO formulation, our training objective maximizes the expected reward while applying KL regularization. Namely, for a given regularization strength \( \beta > 0 \), we optimize:
\begin{equation}\label{eq:DPO_Training_Objective}
\max_{\{w_k, \pi_k\}} \mathbb{E}_{x \sim p,\, y \sim \pi(y\mid x)} \left[ r(x, y) \right]
- \beta \sum_{k=1}^K \mathbb{E}_{x \sim p} \left[ w_k(x) D_{\mathrm{KL}} \left( \pi_k(y \mid x) \,\|\, \pi_{\text{ref}(k)}(y \mid x) \right) \right].
\end{equation}
Since the model supports expert specialization, the regularization term is applied to each expert individually, encouraging proximity to its corresponding reference policy. The gating weights \( w_k(x) \) modulate the strength of this penalty based on the expert's responsibility for each input, thereby promoting prompt- (or more general user-context-) dependent regularization. Moreover, analogous to the original DPO formulation, the combined structure of \eqref{eq:mixpolicy}, \eqref{eq:mix_exp_reward}, and \eqref{eq:DPO_Training_Objective} admits a closed-form solution for the optimal policy in terms of the reward, as shown in Theorem~\ref{lemma:Optimal Expert Policy under KL-Regularized Objective}.

Before deriving specific properties, we formalize the second core component of our model: reward learning consistent with the structure of our mixture preference model. The classical Bradley--Terry (BT) model defines the probability of preferring \( y^+ \) over \( y^- \) in context \( x \) using a single latent reward function \( r(x, y) \):
\[
\mathbb{P}(y^+ \succ y^- \mid x) = \frac{e^{r(x, y^+)}}{e^{r(x, y^+)} + e^{r(x, y^-)}}.
\]
This formulation arises from a logistic likelihood and is typically fit via maximum likelihood estimation over observed preference triplets \( (x_i, y_i^+, y_i^-) \sim \mathcal{D} \):
\[
\max_{r} \sum_i \log \left( \frac{\exp(r(x_i, y^+_i))}{\exp(r(x_i, y^+_i)) + \exp(r(x_i, y^-_i))} \right).
\]

To incorporate the modular structure introduced by the policy~\eqref{eq:mixpolicy}, we extend the classical BT model to account for expert-specific rewards. In particular, since the policy is defined as a mixture over expert policies \( \pi_k \) with gating weights \( w_k(x) \), it is natural to model preferences under a mixture of corresponding reward functions \( r_k \). This leads to the Mixture-of-Bradley--Terry (MBT) model, where the latent expert index \( z \in \{1, \ldots, K\} \) governs which reward function is used to evaluate the pairwise preference.

Conditioned on expert \( z = k \), the preference likelihood follows the standard BT model:
\begin{equation}\label{eq:conditional likelihood under expert k}
\mathbb{P}(y^+ \succ y^- \mid x, z=k) = \frac{e^{r_k(x, y^+)}}{e^{r_k(x, y^+)} + e^{r_k(x, y^-)}}.
\end{equation}
Marginalizing over the latent expert assignment according to the gating distribution \( w_k(x) = p(z = k \mid x) \), we obtain the MBT model:
\[
\mathbb{P}(y^+ \succ y^- \mid x) = \sum_{k=1}^K w_k(x)  \frac{\exp(r_k(x, y^+))}{\exp(r_k(x, y^+)) + \exp(r_k(x, y^-))}.
\]
We learn the reward functions \( \{r_k\} \) by minimizing the negative log-likelihood of the observed preferences under the marginal model. This corresponds to maximum marginal likelihood under the latent-variable formulation of the MBT model:
\begin{equation}\label{eq:MBT_Loss_def}
\mathcal{L}_{\text{MBT}} = - \mathbb{E}_{(x, y^+, y^-) \sim \mathcal{D}} \left[ \log \sum_{k=1}^K w_k(x) \frac{\exp(r_k(x, y^+))}{\exp(r_k(x, y^+)) + \exp(r_k(x, y^-))} \right].
\end{equation}

Direct optimization of the marginal log-likelihood in~\eqref{eq:MBT_Loss_def} leads to high instabilities in the gradients, obscures the latent structure of the model, and provides limited insight into expert specialization. To address this, we derive a variational evidence lower bound (ELBO) that decomposes the marginal likelihood into an expected per-expert log-probability term and a KL divergence between the variational posterior and the expert prior. This decomposition not only facilitates scalable optimization via stochastic gradient methods but also enables interpretation of expert responsibilities in terms of soft assignments over preference pairs.

\begin{theorem}[ELBO for the MBT Model]\label{theorem:ELBO for the MBT Model}
Let \( (x, y^+, y^-) \) be a preference triplet with \( y^+ \succ y^- \), and let \( z \in \{1, \ldots, K\} \) be a latent expert index with prior \( p(z = k \mid x) = w_k(x) \). Let \( \sigma_k(x, y^+, y^-)\) denote the Bradley--Terry likelihood under expert \( k \) given in \eqref{eq:conditional likelihood under expert k}. Then, for any variational distribution \( q(z \mid x, y^+, y^-) \) satisfying \( \sum_k q_k(x, y^+, y^-) = 1 \), we have:
\begin{align*}
\log \mathbb{P}(y^+ \succ y^- \mid x)
&\geq \sum_{k=1}^K q_k(x, y^+, y^-) \log \left( \frac{w_k(x) \, \sigma_k(x, y^+, y^-)}{q_k(x, y^+, y^-)} \right) \\
&= \mathbb{E}_{z \sim q} \left[ \log \sigma_z(x, y^+, y^-) \right] - D_{\mathrm{KL}}(q(z \mid x, y^+, y^-) \,\|\, p(z \mid x)).
\end{align*}
The bound is tight when \( q_k(x, y^+, y^-) = \frac{w_k(x) \, \sigma_k(x, y^+, y^-)}{\sum_{j=1}^K w_j(x) \, \sigma_j(x, y^+, y^-)} \).
\end{theorem}

The following corollary follows by applying Theorem~\ref{theorem:ELBO for the MBT Model} under the joint distribution over preference triplets sampled from the mixture policy. In this setting, the variational posterior \( q(z \mid x, y^+, y^-) \) can be chosen to match the true posterior \( p(z \mid x, y^+, y^-) \), since it admits a closed-form expression and is differentiable with respect to model parameters. This choice makes the KL divergence term in Theorem~\ref{theorem:ELBO for the MBT Model} vanish, resulting in a tight bound that coincides with the log-likelihood. Consequently,
\begin{corollary}[MBT Variational Loss Function]
\label{cor:mbt_loss}
Maximizing the variational lower bound from Theorem~\ref{theorem:ELBO for the MBT Model} over the preference distribution yields the MBT training loss that we aim to minimize:
\begin{align}
\mathcal{L}_{\mathrm{MBT}}
&= - \mathbb{E}_{(x, y^+, y^-) \sim \mathcal{D}}\left[\sum_{k=1}^K q_k(x, y^+, y^-) \log \frac{\exp(r_k(x, y^+))}{\exp(r_k(x, y^+)) + \exp(r_k(x, y^-))}\right]. \label{eq:mbt_loss}
\end{align}
\end{corollary}

\textbf{Reward Decomposition and Expert Alignment:} We now analyze the structure of the objective in~\eqref{eq:DPO_Training_Objective} by decomposing the mixture reward \( r(x, y) \) into expert-specific contributions. This decomposition highlights the connection between policy inference and reward modeling in the presence of latent expert structure. In particular, we define two variational distributions that describe expert responsibilities under the policy and reward components, respectively:
\begin{align}
q^{(\pi)}_k(x, y) = \frac{w_k(x)\pi_k(y \mid x)}{\pi(y \mid x)}, \quad \text{and} \quad
q^{(r)}_k(x, y) = \frac{w_k(x) \exp(r_k(x, y))}{\sum_j w_j(x) \exp(r_j(x, y))}. \label{eq:reward_posterior}
\end{align}

The distribution \( q^{(\pi)} \) corresponds to the posterior over experts under the mixture policy, while \( q^{(r)} \) arises from the softmax aggregation of expert rewards. The following result expresses the mixture reward exactly in terms of the policy-induced posterior and its divergence from the reward-induced posterior.
\begin{theorem}[Equality Decomposition of Reward Mixture]\label{theorem:Equality Decomposition of Reward Mixture}
Let \( q^{(\pi)} \) and \( q^{(r)} \) be defined as in \eqref{eq:reward_posterior}. Then the mixture reward satisfies the following exact decomposition:
\begin{align}
r(x, y)
&= \sum_{k=1}^K q^{(\pi)}_k(x, y) \left[ r_k(x, y) + \log w_k(x) - \log q^{(\pi)}_k(x, y) \right] \notag \\
&\quad + D_{\mathrm{KL}}\left(q^{(\pi)}(\cdot \mid x, y) \, \| \, q^{(r)}(\cdot \mid x, y)\right). \label{eq:reward_decomposition}
\end{align}
\end{theorem}
The decomposition in Theorem~\ref{theorem:Equality Decomposition of Reward Mixture} reveals that the reward in \eqref{eq:mix_exp_reward} can be expressed as an expectation under the policy-induced expert posterior \( q^{(\pi)} \), corrected by a KL divergence term that quantifies the mismatch between the expert responsibilities inferred by the policy and those implied by the reward scores. Unlike in standard variational inference settings—where the variational posterior is freely optimized—\( q^{(\pi)} \) is fully determined by the policy parameters \( \{w_k, \pi_k\} \) and cannot be adjusted to match \( q^{(r)} \). Consequently, the KL term does not vanish and must be retained as part of the objective. It plays a critical role in regularizing modular learning by penalizing structural misalignment between policy inference and reward modeling, thus enabling a principled treatment of expert specialization.

Next, we utilize the decomposition within Theorem~\ref{theorem:Equality Decomposition of Reward Mixture} to provide the decomposition for the policy training objective in~\eqref{eq:DPO_Training_Objective}. Namely,
\begin{lemma}[MoE-DPO Objective Decomposition]\label{lemma:MoE-DPO Objective Decomposition}
The training objective in~\eqref{eq:DPO_Training_Objective} can be written as
\begin{align}
\mathcal{L}_{\text{MoE-DPO}} &= \sum_{k=1}^K \mathbb{E}_{x \sim p,\, y \sim \pi_k(\cdot \mid x)} \left[ w_k(x)  \left( \widetilde{r}_k(x, y)   - \beta \log \frac{\pi_k(y \mid x)}{\pi_{\text{ref}(k)}(y \mid x)} \right) \right]. \label{eq:L_MoEDPO}
\end{align}
where \( \widetilde{r}_k(x, y) = r_k(x, y) - \log \left( \frac{q_k^{(r)}(x, y)}{w_k(x)} \right) \).
\end{lemma}
The extra term in \( \widetilde{r}_k(x, y) \) adjusts the reward to account for discrepancies between the gating distribution \( w_k(x) \) and the posterior responsibilities \( q_k^{(r)}(x, y) \) induced by the reward model. This term acts as a localized KL penalty that encourages alignment between policy-induced expert selection and reward-based expert attribution, promoting consistency between modular inference and modular supervision. This formulation admits closed-form updates for each expert policy \( \pi_k \), facilitating scalable and modular optimization. Specifically, the per-expert KL-regularized objective becomes:
\[
\mathcal{L}_k(\pi_k) = \mathbb{E}_{x \sim p} \left[ w_k(x) \, \mathbb{E}_{y \sim \pi_k(y \mid x)} \left[\widetilde{r}_k(x,y) - \beta \log \frac{\pi_k(y \mid x)}{\pi_{\text{ref}(k)}(y \mid x)}\right] \right],
\]
and the optimal solution takes the form detailed in the next theorem.
\begin{theorem}[Policy-Reward Equivalence under Mixture Models]\label{lemma:Optimal Expert Policy under KL-Regularized Objective}
The expert policy that maximizes \( \mathcal{L}_k(\pi_k) \) is given by:
\[
\pi_k^*(y \mid x) = \frac{1}{Z'_k(x)} \pi_{\text{ref}(k)}(y \mid x) \exp\left( \frac{1}{\beta} r_k(x, y) \right),
\]
where the normalizing constant is defined as \( Z_k(x) = \sum_y \pi_{\text{ref}(k)}(y \mid x) \exp\left( \frac{1}{\beta} \widetilde{r}_k(x, y) \right) \).
\end{theorem}

This closed-form update mirrors the structure of the standard DPO solution while incorporating expert-specific responsibilities and reward corrections. It supports independent optimization of each expert, making the overall training procedure efficient and amenable to parallelization.

Since the latent variable \( z \) is shared across the preference-alignment and policy-optimization models, we can leverage the closed-form correspondence between the expert-specific reward \( r_k(x, y) \) and the optimal policy \( \pi_k(y \mid x) \), as established in Theorem~\ref{lemma:Optimal Expert Policy under KL-Regularized Objective}. Substituting the reward expression,
\begin{equation}\label{eq:reward_policy}
r_k(x,y) = \beta \log\left( \frac{\pi_k(y \mid x) Z_k(x)}{\pi_{\text{ref}(k)}(y \mid x)} \right) + \log\left( \frac{q^{(r)}_k(x,y)}{w_k(x)} \right),
\end{equation}
into the MBT variational bound in~\eqref{eq:mbt_loss} yields a reward-modulated loss that is decomposable across expert heads. Specifically, we express the total MBT loss as \( \mathcal{L}_{\mathrm{MBT}} = \sum_{k=1}^K \mathcal{L}_k^{\mathrm{MBT}}(\pi_k) \),
where each expert-specific objective \( \mathcal{L}_k^{\mathrm{MBT}} \) is given by
\begin{align}
\mathcal{L}_k^{\mathrm{MBT}}(\pi_k) =
- \mathbb{E}_{(x, y^+, y^-) \sim \mathcal{D}}
\left[
q_k(x, y^+, y^-)
\log \frac{
\left( \frac{\pi_k(y^+ \mid x)}{\pi_{\mathrm{ref}(k)}(y^+ \mid x)} \right)^\beta  q^{(r)}_k(x, y^+)
}{
\sum_{\eta \in \{y^+, y^-\}}
\left( \frac{\pi_k(\eta \mid x)}{\pi_{\mathrm{ref}(k)}(\eta \mid x)} \right)^\beta q^{(r)}_k(x, \eta)
}
\right],\label{eq:DPO_k_policy_loss}
\end{align}
where \( q_k(x, y^+, y^-) \) is defined in Theorem~\ref{theorem:ELBO for the MBT Model}.
This decomposition isolates the contribution of each expert head \( \pi_k \), allowing for fully decoupled and independent optimization of \( \pi_k \) parameters. However, note that the original DPO model and training objective from~\cite{rafailov2023direct} are recovered when \( K = 1 \).

Finally, to derive the optimization objective for the prior weights \( w_k(x) \), we first note that, from Theorem~\ref{theorem:ELBO for the MBT Model}, the only dependence of the ELBO on \( w_k(x) \) appears in the KL divergence term between the variational posterior \( q_k(x, y^+, y^-) \) and the prior \( w_k(x) \). Consequently, we show in the Appendix how maximizing the ELBO with respect to \( w_k(x) \) is equivalent to the minimization
\begin{equation}\label{eq:wk_loss}
\argmin_{w_k(x)}\mathbb{E}_{(x, y^+, y^-) \sim \mathcal{D}} \left[ \sum_{k=1}^K q_k(x, y^+, y^-) \log \frac{q_k(x, y^+, y^-)}{w_k(x)} \right].
\end{equation}
This objective encourages weights to match the empirical expert responsibilities inferred from preference data, ensuring that the prior over experts reflects their posterior under the MBT model. Next, we consider two variants of our algorithm depending on the structure of the mixture weights.\\
\textbf{Mix-DPO.} The optimal weights \( w_k(x) \) are fixed across inputs, and their update from~\eqref{eq:wk_loss} has a closed form given by averaging the responsibilities:
\( w_k \gets \frac{1}{n} \sum_{i=1}^n q_k(x_i, y_i^+, y_i^-) \), and then renormalizing.\\
\textbf{MoE-DPO.} The weights are input-dependent and parameterized by a gating network \( w_k(x; \phi) \) via a softmax over logits. The parameters \( \phi \) are updated by minimizing the objective in~\eqref{eq:wk_loss}, which reduces to the cross-entropy between predicted weights and inferred posteriors:
\( \mathcal{L}_{\text{gating}}(\phi) = - \mathbb{E}_{(x, y^+, y^-) \sim \mathcal{D}} \left[ \sum_k q_k(x, y^+, y^-) \log w_k(x; \phi) \right] \).
The gating function can also be conditioned on user metadata \( u \in \mathcal{U} \), enabling personalized mixture weights \( w_k(x, u; \phi) \) without modifying the expert policies.

\section{Algorithm}
\label{sec:algorithm}
Building on the variational framework and expert-specific loss decomposition in Section~\ref{sec:model}, we now present the training procedure for optimizing modular preference-aligned policies. As shown in Corollary~\ref{cor:mbt_loss} and Theorem~\ref{lemma:Optimal Expert Policy under KL-Regularized Objective}, the MoE-DPO objective admits a per-expert decomposition, where each expert policy \( \pi_k \) is trained by minimizing its own KL-regularized preference loss \( \mathcal{L}_k^{\mathrm{MBT}} \). To ensure consistency between the policy and reward components, expert rewards are updated after each policy step via the closed-form transformation in~\eqref{eq:reward_policy} derived from the optimality condition.

Algorithm~\ref{alg:EM_MBT_short} summarizes the resulting variational EM procedure. Each iteration alternates between:
(i) an E-step that computes expert responsibilities \( q_k(x, y^+, y^-) \) using the MBT posterior in Theorem~\ref{theorem:ELBO for the MBT Model};
(ii) an M-step that updates each expert policy \( \pi_k \) by minimizing the loss \( \mathcal{L}_k^{\mathrm{MBT}} \) in~\eqref{eq:DPO_k_policy_loss};
(iii) a reward update via the closed-form expression in~\eqref{eq:reward_policy}; and
(iv) an update of the mixture weights \( w_k(x) \) based on inferred responsibilities and the objective function in~\eqref{eq:wk_loss}.
This modular structure allows flexible optimization regimes: depending on the setting, the expert policies, the gating network, or both can be updated jointly or independently within each iteration.

\begin{algorithm}[]
\caption{Variational EM Algorithm for Mixture Preference Alignment}
\label{alg:EM_MBT_short}
\begin{algorithmic}[1]
\Require Expert policies and rewards \(\{\pi_k, r_k\}_{k=1}^K\), mixture weights \( w_k(x) \), reference policies \( \pi_{\text{ref}(k)} \), temperature \( \beta \).
\While{not converged}
    \State \textbf{Sample minibatch:} \( \{(x_i, y_i^+, y_i^-)\}_{i=1}^n \sim \mathcal{D} \).
    \If{Trainable expert policies}
        \State \textbf{E-step:} Compute posteriors \( q_k(x_i, y_i^+, y_i^-) \) from rewards \( r_k \) and weights \( w_k(x) \).
        \State \textbf{M-step:} Update each \( \pi_k \) using \( q_k \), \( \pi_{\text{ref}(k)} \), and temperature \( \beta \).
    \EndIf
    \State \textbf{Reward update:} Update each \( r_k(x, y) \) using Eq.~\eqref{eq:reward_policy}.
    \If{Trainable mixture weights}
        \State \textbf{E-step:} Recompute responsibilities \( q_k(x_i, y_i^+, y_i^-) \).
        \State \textbf{M-step:} Update gating network parameters \( \phi \) (MoE-DPO) or fixed weights (Mix-DPO) \( w_k \).
    \EndIf
\EndWhile
\end{algorithmic}
\end{algorithm}

This algorithm implements the full variational EM cycle over preference-labeled data. Expert policies and rewards are jointly updated in a way that maintains the policy–reward correspondence (Theorem~\ref{lemma:Optimal Expert Policy under KL-Regularized Objective}), while the mixture weights adapt to match inferred responsibilities. The structure supports parallelization across experts and modular personalization via user-conditioned gating.

\section{Experimental Evaluation}\label{sec:experiments}
This section details the experimental evaluation of our proposed Mix- and MoE-DPO methods against baseline DPO. We first present two primary experiments on the alignment of GPT-2 \cite{radford2019language} for the review generation task \cite{rafailov2023direct}, each exploring variations of the task: (1) a multi-reward task, where the model is trained to generate positive, informative, and grammatically correct movie reviews, and (2) a contextual preference-alignment task, where the model is trained to generate either positive movie or book reviews based on the given prompt. We provide ablation studies to assess the impact of fixed mixture weights or expert policies, which we include as an optional simplification within Algorithm~\ref{alg:EM_MBT_short}.

\subsection{Mix-DPO for Multi-Reward Movie Review Generation}
\textbf{Models:} We apply Mix-DPO for preference alignment to the three different reward functions of sentiment, informativeness, and grammar using a supervised-fine-tuned GPT-2 \cite{radford2019language} on the IMDb \cite{maas-etal-2011-learning} dataset as the base and reference policy. We evaluate two configurations: Case 1 uses a single GPT-2 with three heads (final linear layers) for parameter-efficient sharing, and Case 2 employs three copies of the fine-tuned GPT-2 model. \textbf{Datasets:} We construct a pairwise preference dataset from IMDb similar to \cite{rafailov2023direct}, where we first sample completions for the 25{,}000 prompts in the training dataset. Preferences over these responses are annotated with respect to sentiment, informativeness, and grammar, with further details on our reward-function construction included in the Appendix. \textbf{Algorithms:} Mix-DPO minimizes a weighted loss across mixture components with learnable weights initialized at \(1/3\) for both cases. For Case 1 (parameter-efficient sharing), training proceeds for 3 epochs, while for Case 2 (independent models), training proceeds for 1 epoch, both with a learning rate of \(10^{-5}\), batch size of 64, and the AdamW optimizer. \textbf{Metrics:} We generate completions from 10{,}000 prompts in the test set and compute reward scores corresponding to each ground-truth reward function. Completions can be generated via individual mixture components or the entire mixture, with both sets of metrics reported in Table~\ref{tab:exp1_rewards}, in addition to baseline DPO.

\textbf{Results:} During training of Mix-DPO, we log the posterior weights \( q_k(x_i, y_i^+, y_i^-) \) corresponding to each head and annotation style. In Figure~\ref{fig:training} (left panel), we see a clear separation of \(q\)-weights in each head based on annotation style, indicating specialization of heads for the distinct reward functions of sentiment, informativeness, and grammar. Specifically, we observe that, averaged over data from the corresponding annotation style, the posterior weights diverge over training, with Head 0 decreasing to approximately 0.31, Head 1 stabilizing around 0.34, and Head 2 increasing to 0.35.

\begin{figure}[!htp]
  \centering
  \begin{minipage}{0.47\textwidth}
    \centering
    \includegraphics[width=\textwidth,height=4cm,keepaspectratio]{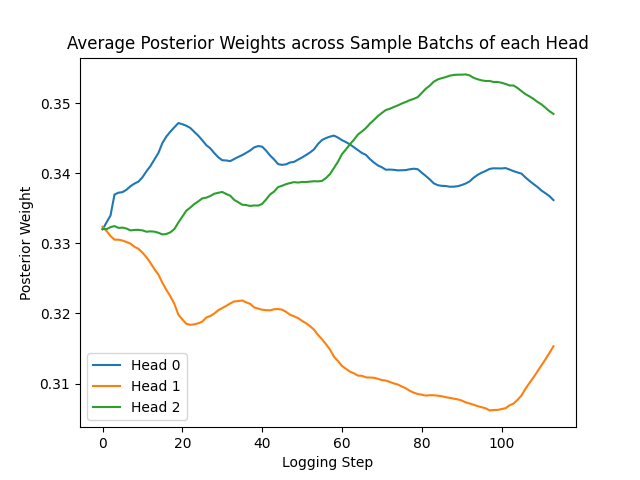}
  \end{minipage}\hspace{0.005\textwidth}
  \begin{minipage}{0.47\textwidth}
    \centering
    \includegraphics[width=\textwidth,height=4cm,keepaspectratio]{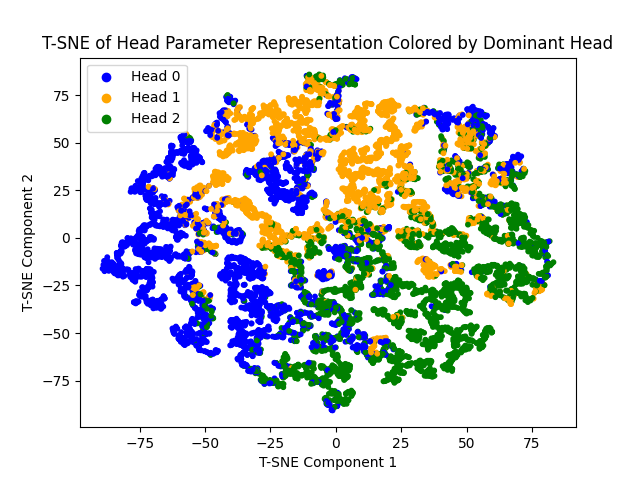}
  \end{minipage}\hspace{0.005\textwidth}
  \caption{\textbf{Left panel:} Average posterior weights for Mix-DPO heads—(a) \textcolor{blue}{Head 0}, (b) \textcolor{orange}{Head 1}, and (c) \textcolor{Green}{Head 2}—indicates specialization. \textbf{Right panel:} t-SNE plot of head-parameter representations indicates head separation.}\label{fig:training}
\end{figure}

On the test set, Table~\ref{tab:exp1_rewards} shows that Case 1 (Mixture effectiveness in aligning LLMs) outperforms baseline DPO across sentiment (0.654 ± 0.004 vs.\ 0.610 ± 0.004) and grammar (0.241 ± 0.001 vs.\ 0.216 ± 0.001), but underperforms on informativeness (0.326 ± 0.007 vs.\ 0.363 ± 0.008). Ablation studies in Table~\ref{tab:exp1_rewards} show that fixed weights improve over DPO but do not perform as well as Mix-DPO when averaging heads.

\begin{table}[!htp]
\centering
\small
\caption{Reward scores (Mean ± SE) for Mix-DPO on the IMDb test set.}
\label{tab:exp1_rewards}
\begin{tabular}{lccc}
\toprule
Model & \textcolor{blue}{Sentiment} & \textcolor{orange}{Informativeness} & \textcolor{Green}{Grammar} \\[0.1cm]
\midrule
Baseline DPO & 0.610 ± 0.004 & 0.363 ± 0.008 & 0.216 ± 0.001 \\[0.1cm]
\toprule
Case 1 (Mixture) & 0.654 ± 0.004 & 0.326 ± 0.007 & 0.241 ± 0.001 \\
Case 1 (Sparse)  & & & \\
\quad Head 0 & 0.616 ± 0.020 & \textcolor{orange}{\textbf{0.396}} ± 0.007 & 0.263 ± 0.001 \\[0.1cm]
\quad Head 1 & \textcolor{blue}{\textbf{0.720}} ± 0.003 & 0.394 ± 0.008 & 0.213 ± 0.001 \\[0.1cm]
\quad Head 2 & 0.632 ± 0.004 & 0.342 ± 0.007 & \textcolor{Green}{\textbf{0.267}} ± 0.001 \\[0.1cm]
\toprule
Case 2 (Mixture) & \textbf{0.664} ± 0.004 & \textbf{0.350} ± 0.006 & \textbf{0.465} ± 0.002 \\
\toprule
Fixed Weights (\(w_i = 1/3\)) & 0.646 ± 0.004 & 0.318 ± 0.009 & 0.239 ± 0.002 \\
\bottomrule
\end{tabular}
\end{table}

A word cloud of sampled terms from the heads (Figure~\ref{fig:word_cloud}) highlights their specialization: Head 0—strong in informativeness (0.396 ± 0.007) and grammar (0.263 ± 0.001), but weaker in positive sentiment (0.616 ± 0.002)—produces terms like “\textcolor{blue}{Lebowski},” “\textcolor{blue}{Solomon},” and “\textcolor{blue}{average}”; Head 1—excelling in positive sentiment (0.720 ± 0.003) and informativeness (0.394 ± 0.008)—generates terms like “\textcolor{orange}{funny},” “\textcolor{orange}{good},” and “\textcolor{orange}{Cusack},” reflecting a focus on emotional tone; and Head 2—with the highest grammar score (0.267 ± 0.001)—focuses on syntactic structure with terms like “\textcolor{Green}{movie},” “\textcolor{Green}{interesting},” and “\textcolor{Green}{story}.”

\begin{figure}[!htp]
  \centering
  \begin{minipage}{0.31\textwidth}
    \centering
    \includegraphics[width=\textwidth,height=8cm,keepaspectratio]{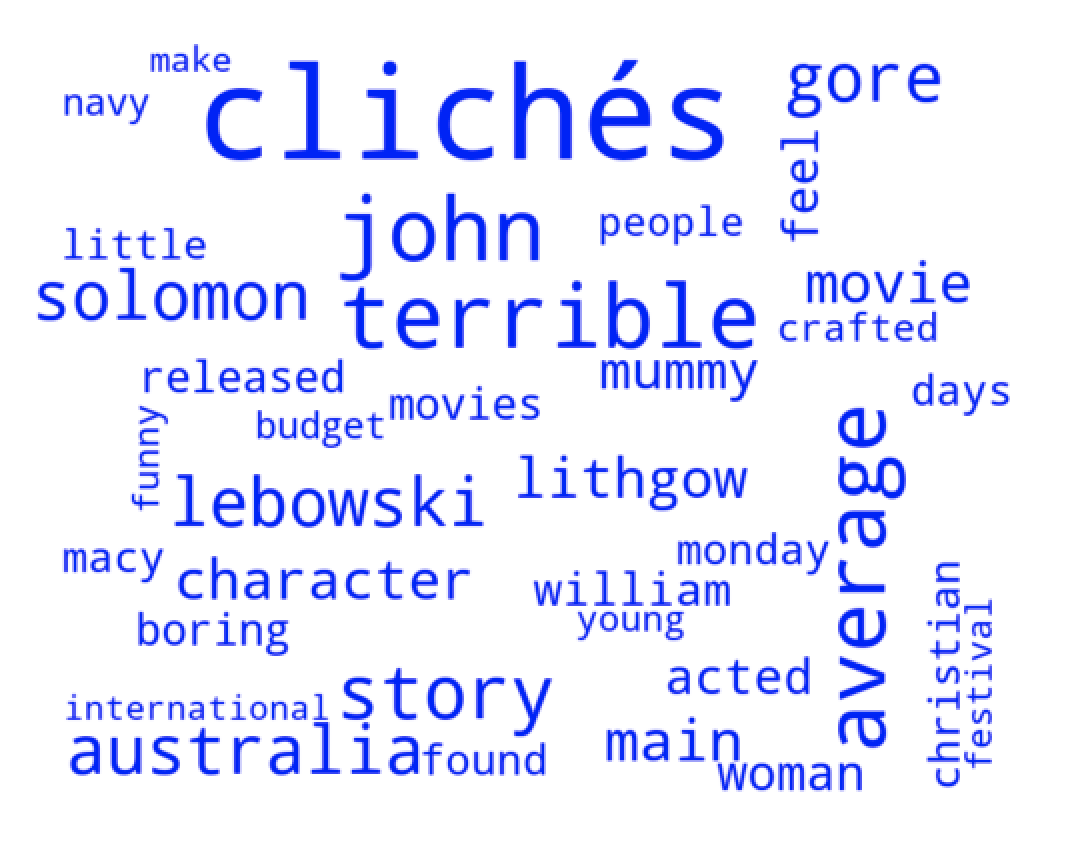}
  \end{minipage}\hspace{0.015\textwidth}
  \begin{minipage}{0.31\textwidth}
    \centering
    \includegraphics[width=\textwidth,height=8cm,keepaspectratio]{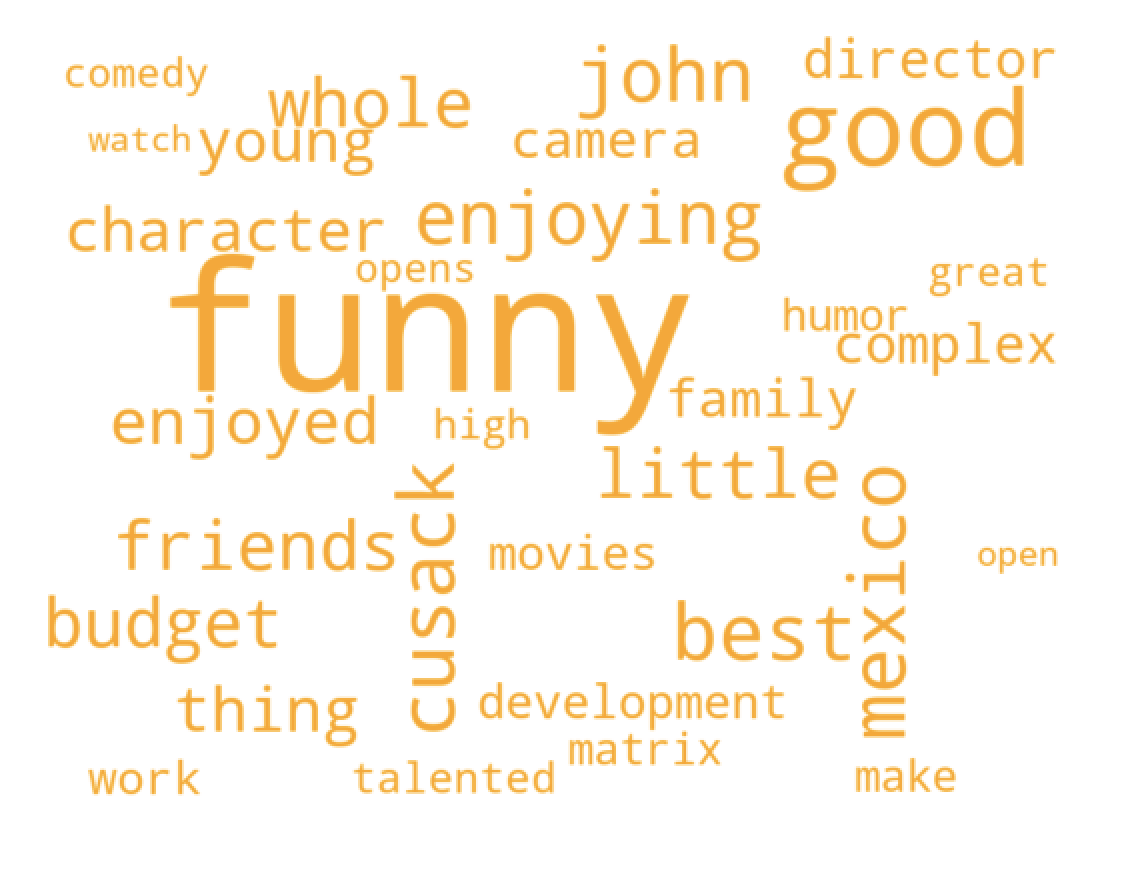}
  \end{minipage}\hspace{0.015\textwidth}
  \begin{minipage}{0.31\textwidth}
    \centering
    \includegraphics[width=\textwidth,height=8cm,keepaspectratio]{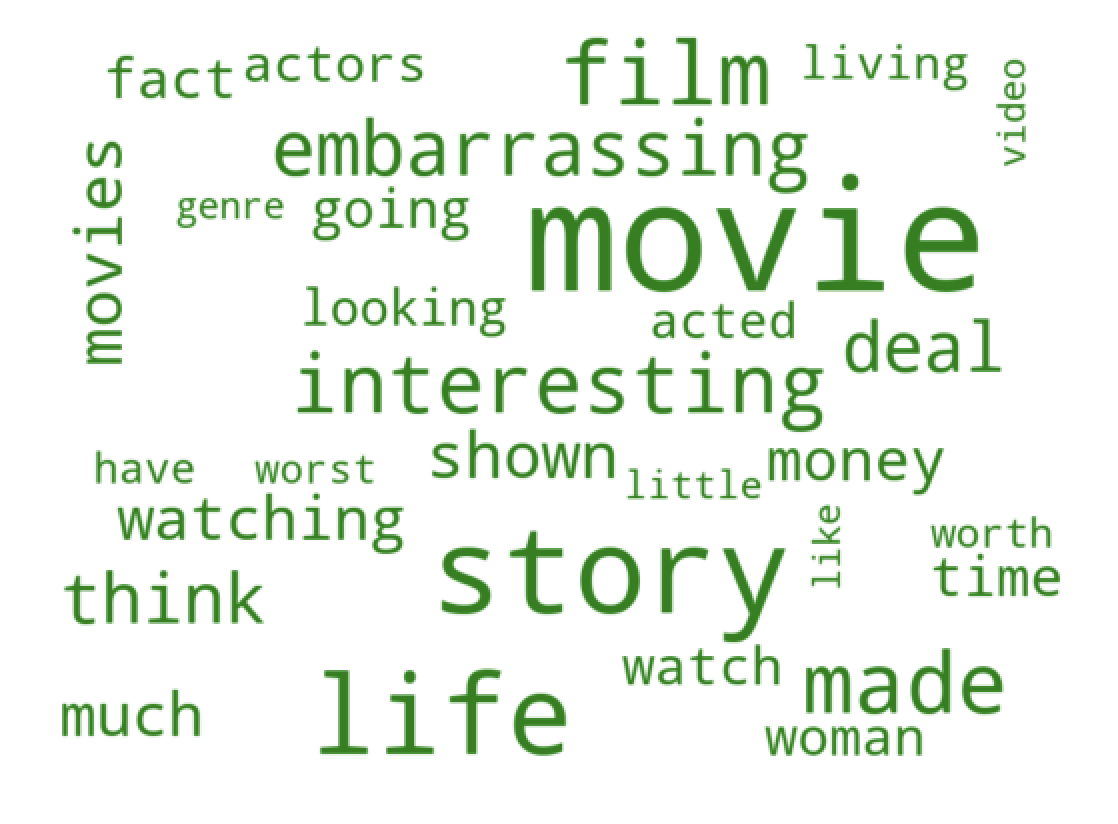}
  \end{minipage}
  \caption{Sampled words from responses generated by Mix-DPO heads in Case 1: (a) \textcolor{blue}{Head 0} (Informativeness and Grammar), (b) \textcolor{orange}{Head 1} (Positive Sentiment and Informativeness), and (c) \textcolor{Green}{Head 2} (Grammar).}
  \label{fig:word_cloud}
\end{figure}

\subsection{MoE-DPO for Multi-Task Review Generation}
\textbf{Models:} We apply MoE-DPO for preference alignment to two different user groups (i.e., tasks) of book reviews and movie reviews. Using the same GPT-2 base model~\cite{radford2019language}, we evaluate our algorithm in the case of parameter sharing (Case 1) and independent models (Case 2). For MoE-DPO, we utilize prompt-dependent weights via a pretrained linear classifier—i.e., \( w(x)=Wx+b \)—which can be frozen or jointly learned during training. \textbf{Datasets:} We augment the IMDb dataset \cite{maas-etal-2011-learning} with 25{,}000 pairwise preferences from Amazon Book Reviews, a subset of the Amazon review dataset \cite{hou2024bridging}, forming a 50{,}000-pair dataset with prompts labeled for movie or book reviews. Further details on the book-review dataset can be found in the Appendix. We pretrain a linear classifier for 5{,}000 steps on the prompts to achieve approximately \(65\%\) accuracy for the true source label. \textbf{Algorithms:} We employ the mixture-of-experts loss for both cases, and in joint learning we alternate between the two losses \( \mathcal{L}_{\text{MoE-DPO}} \) and \( \mathcal{L}_{\text{gating}} \). Training proceeds for 3 epochs with a learning rate of \(10^{-5}\), batch size of 64, and the AdamW optimizer. \textbf{Metrics:} We generate completions for 10{,}000 test-set prompts, equally split between 5{,}000 book and 5{,}000 movie reviews. Reward scores are computed with the ground-truth sentiment reward function.

\textbf{Results:} The performance of MoE-DPO across different configurations is summarized in Table~\ref{tab:exp2_rewards}. In Table~\ref{tab:exp2_rewards}, Case 1 (Frozen Gating Layer) achieves a sentiment reward on the movie task of 0.638 ± 0.005, on par with the learnable gating layer at 0.639 ± 0.005. For the book task, a learnable gating layer outperforms, with 0.734 ± 0.004 compared to 0.709 ± 0.004 achieved with a frozen gating layer. Both Case 1 configurations outperform baseline DPO (0.603 ± 0.005 on movie sentiment and 0.648 ± 0.005 on book sentiment), indicating that improved multi-task alignment can be achieved via our mixture model.

\begin{table}[!htp]
\centering
\small
\caption{Sentiment reward scores (Mean ± SE) for MoE-DPO.}
\label{tab:exp2_rewards}
\begin{tabular}{lcc}
\toprule
Model & Movie Sentiment & Book Sentiment \\[0.1cm]
\midrule
Baseline DPO & 0.603 ± 0.005 & 0.648 ± 0.005 \\[0.1cm]
\midrule
Case 1 (Frozen Gating Layer) & \textbf{0.638} ± 0.005 & 0.709 ± 0.004 \\
Case 1 (Learnable Gating Layer) & \textbf{0.639} ± 0.005 & \textbf{0.734} ± 0.004 \\
\bottomrule
\end{tabular}
\end{table}

\begin{figure}[!htp]
  \centering
  \begin{minipage}{0.47\textwidth}
    \centering
    \includegraphics[width=\textwidth,height=4cm,keepaspectratio]{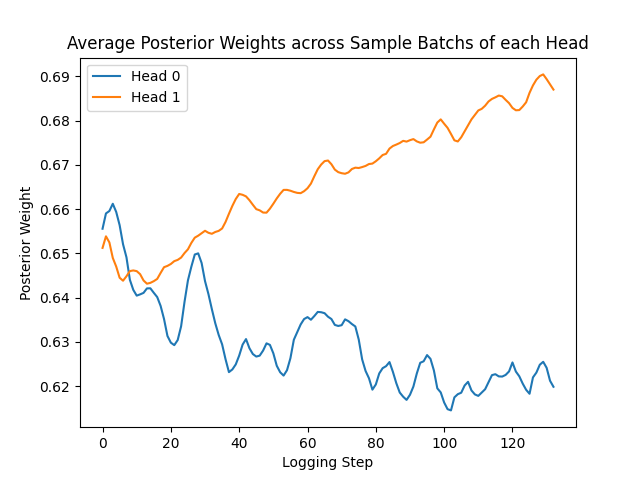}
  \end{minipage}\hspace{0.005\textwidth}
  \begin{minipage}{0.47\textwidth}
    \centering
    \includegraphics[width=\textwidth,height=4cm,keepaspectratio]{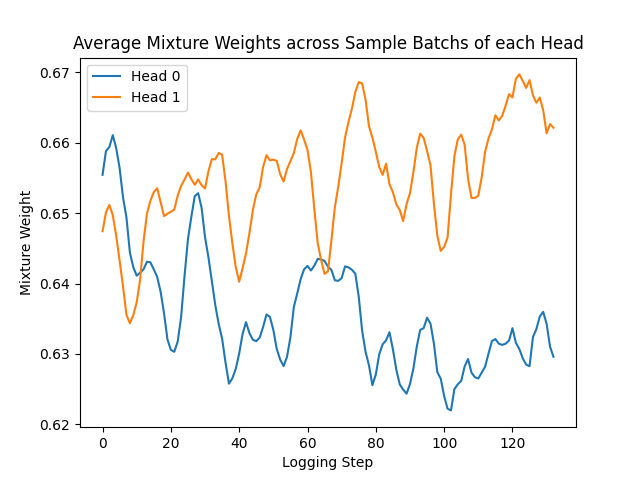}
  \end{minipage}\hspace{0.005\textwidth}
  \caption{\textbf{Left panel:} Average posterior weights for MoE-DPO heads in Case 1—(a) \textcolor{blue}{Head 0} and (b) \textcolor{orange}{Head 1}—indicates specialization. \textbf{Right panel:} Average mixture weights for MoE-DPO indicates prompt separation.}
  \label{fig:moe_weights}
\end{figure}

\begin{figure}[!htp]
  \centering
  \begin{minipage}{0.31\textwidth}
    \centering\includegraphics[width=1.1\textwidth,height=4cm]{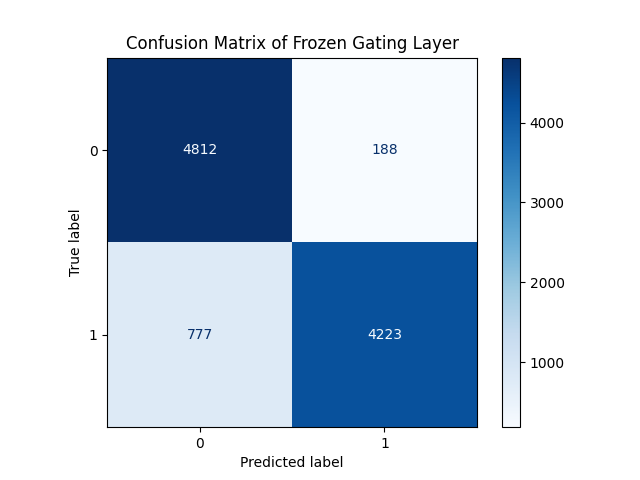}
  \end{minipage}
  \begin{minipage}{0.31\textwidth}
    \centering
    \includegraphics[width=1.1\textwidth,height=4cm]{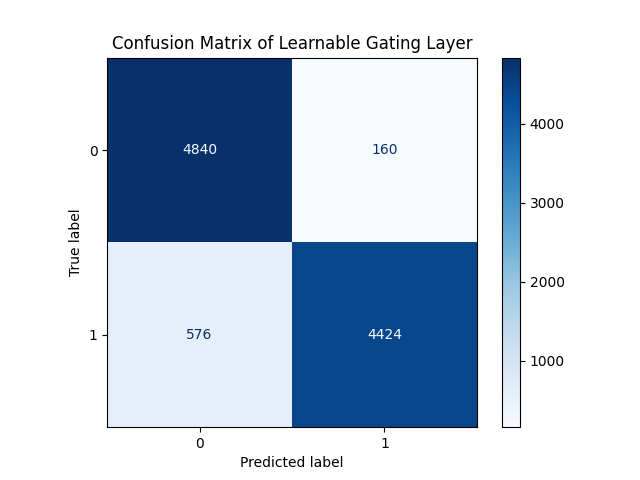}
  \end{minipage}
  \begin{minipage}{0.31\textwidth}
    \centering
    \includegraphics[width=1.1\textwidth,height=4cm]{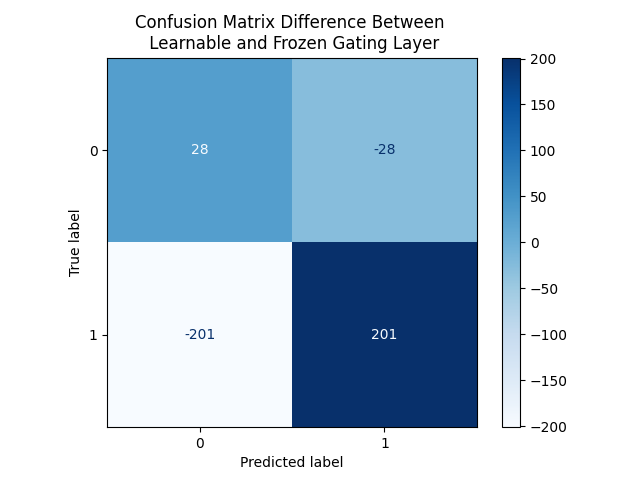}
  \end{minipage}
  \caption{Confusion matrices of frozen and learnable gating layers for movie (0) vs.\ book (1) prompts, with the difference indicating improvements in predicted labels under joint learning during training.}
  \label{fig:moe_confusion}
\end{figure}

\section{Conclusion}
Mix- and MoE-DPO provides a unified approach to modular preference alignment. It generalizes DPO with latent expert mixtures, supports efficient variational training, and enables expert reuse, contextual adaptation, and user-specific specialization—advancing the practical and theoretical foundation for aligning LLMs. Beyond the core formulation presented in this paper, the proposed variational framework supports several extensions that broaden its applicability and scalability. These include (i) user-personalized gating, (ii) multi-agent temporal modeling via hidden Markov mixtures, (iii) multimodal alignment through structured group actions, and (iv) scalable training via sparse expert activation and differentiable relaxation (Monte Carlo relaxation of the variational EM). These extensions make Mix- and MoE-DPO well suited for real-world deployment in (i) multi-user, (ii) multi-agent, (iii) multimodal, and (iv) large-scale training settings, respectively.

\bibliography{ref}

 \appendix

\section{Related Work}
Recent extensions have built upon the direct optimization approach of DPO to further improve effectiveness of aligning LLMs with human preferences. Improved robustness and efficiency are achieved with methods that dynamically adjust \( \beta \) to handle noisy preference data, or filter low-quality preference pairs to strengthen optimization \cite{wu2024,morimura2024filtered}. The sample-efficient DPO extension of \( \chi^2 \)-based preference optimization has been proposed to mitigate over-optimization in \cite{huang2024correcting}. Robustness has been further improved by integrating the supervised fine-tuning loss as an implicit adversarial regularizer, or by modifying the underlying preference objective to mitigate overfitting and preference shifts, respectively, in \cite{liu2024provably,azar2024general}. Additional extensions of DPO include \cite{ethayarajh2024kto,hong2024orpo,guo2024direct}.

The precision and applicability of DPO have also been advanced through targeted extensions. Fine-grained factuality alignment and error reduction in multi-turn dialogues are achieved by optimizing sentence-level preferences and adapting DPO for conversational settings, respectively \cite{gu2025mask,shi2024direct}. Alternative structures—such as reasoning steps or token-level alignment—have been integrated into direct optimization approaches in \cite{lai2024step,zeng2024token}. These contributions collectively enhance DPO’s robustness, efficiency, precision, and applicability in addressing alignment challenges.

\section{Architectural Variants of Mix- and MoE-DPO}
\label{appdx:architectures}

To accommodate diverse deployment scenarios, we consider two primary architectural instantiations of the Mix- and MoE-DPO framework. These variants differ in the degree of parameter sharing among expert policies, balancing expressivity, memory efficiency, and specialization. Both designs are compatible with the core variational formulation in Section~\ref{sec:model} and the training algorithm in Section~\ref{sec:algorithm}.

\subsection{Case 1: Shared Encoder with Expert-Specific Heads}
\label{appdx:shared_encoder}

This parameter-efficient variant builds on a modular architecture in which all experts share a common encoder \( f_{\phi}(x, y) \), parameterized by \( \phi \), that produces a joint representation of the input–output pair. Each expert \( k \) is defined by a head-specific transformation \( h_{\psi_k} \) producing logits for the conditional distribution:
\[
\pi_k(y \mid x) = \text{softmax}(h_{\psi_k}(f_{\phi}(x, y))).
\]
The mixture policy is given by the gating-weighted average:
\[
\pi(y \mid x) = \sum_{k=1}^K w_k(x) \cdot \pi_k(y \mid x),
\]
where the gating weights \( \{w_k(x)\} \) form a probability simplex over experts, as in~\eqref{eq:mixpolicy}. KL regularization is applied using a shared reference policy \( \pi_{\mathrm{ref}(k)} = \pi_{\mathrm{ref}} \) for all experts, consistent with~\eqref{eq:DPO_Training_Objective}.

This design offers favorable memory and compute scaling by replacing \( K \) independent policies with one shared encoder and \( K \) lightweight heads. Optimization can proceed jointly over \( (\phi, \{\psi_k\}, \{w_k\}) \), or in a partially frozen regime where \( \phi \) is fixed and only the heads and gating parameters are trained:
\[
\mathcal{L}_{\text{MoE-DPO}}(\{\psi_k\}, \{w_k\} \mid \phi).
\]
This setup aligns with the multi-head specialization strategy in Section~\ref{sec:experiments} and supports modular adaptation over diverse tasks without duplicating the full backbone.

\subsection{Case 2: Fully Independent Experts}
\label{appdx:independent_experts}

In this more expressive variant, each expert policy \( \pi_k(y \mid x) \) is independently parameterized with its own encoder, decoder, and reward function \( r_k(x, y) \). The gating function \( w_k(x) \) may be fixed (as in Mix-DPO) or learned (as in MoE-DPO), and each expert may have a distinct reference policy \( \pi_{\mathrm{ref}(k)} \). This structure supports heterogeneity in both modeling and supervision, allowing for maximal task-specific adaptation.

The decomposition of the MoE-DPO loss into per-expert objectives, as shown in~\eqref{eq:DPO_k_policy_loss}, allows for fully decoupled optimization:
\[
\mathcal{L}_{\text{MoE-DPO}} = \sum_{k=1}^K \mathcal{L}_k^{\text{MBT}}(\pi_k),
\]
with each \( \pi_k \) updated independently. This independence facilitates training under domain shifts, user-specific feedback distributions, or multi-objective alignment, as explored in Section~\ref{sec:experiments}. The associated reward updates and posterior responsibilities follow the closed-form expressions derived in~\eqref{eq:reward_policy} and~\eqref{eq:reward_posterior}, respectively.

\subsection{Hybrid Architectures and Dynamic Expert Routing}

The generality of Mix- and MoE-DPO supports hybrid configurations that combine shared and independent components. For instance:
\begin{itemize}
    \item A shared encoder with \( K_s \) expert heads (Case 1) may be combined with \( K_i \) independently parameterized experts (Case 2), allowing coarse-grained specialization through routing.
    \item A mixture-of-mixtures architecture can be used in which gating weights distinguish between shared and independent blocks, and the overall policy is a convex combination across both groups.
    \item Personalized or task-conditioned routing can be implemented by conditioning \( w_k(x) \) on metadata, enabling efficient multi-user adaptation (cf.\ Appendix~\ref{sec:extended_mix_moe_dpo}).
\end{itemize}

Such hybrid structures offer favorable trade-offs between parameter reuse and specialization, making them well suited for deployment under memory constraints, on-device adaptation, or transfer learning from expert checkpoints. The modularity of our variational framework ensures compatibility with these extensions, including user-aware gating, sparse activation, and scalable training, as discussed in Appendix~\ref{sec:extended_mix_moe_dpo} and supported by the convergence analysis in Appendix~\ref{sec:convergence_analysis}.

 \section{Proofs}\label{appdx:main_proofs}

\begin{lemma}\label{lemma:variationalIdentity}
Let \( a_1, \dots, a_K \in \mathbb{R}_{>0} \) be positive real numbers, and let \( q \in \Delta_K \) be any probability distribution over \( \{1, \dots, K\} \), where
\[
\Delta_K := \left\{ q \in \mathbb{R}^K \,\middle|\, q_k \geq 0,\ \sum_{k=1}^K q_k = 1 \right\}.
\]
Then the following identity holds:
\[
\log \sum_{k=1}^K a_k = \sum_{k=1}^K q_k \log \left( \frac{a_k}{q_k} \right) + D_{\mathrm{KL}}\!\left( q \,\|\, \frac{a}{\sum_j a_j} \right),
\]
where \( \frac{a}{\sum_j a_j} \in \Delta_K \) is the normalized vector with components
\[
p_k := \frac{a_k}{\sum_j a_j}, \quad \text{and} \quad D_{\mathrm{KL}}(q \,\|\, p) = \sum_k q_k \log \frac{q_k}{p_k}.
\]
\end{lemma}

\begin{proof}
Let \( a_k > 0 \) for each \( k \in \{1, \dots, K\} \), and fix any \( q \in \Delta_K \). Define the normalized version of \( a \) as
\[
p_k := \frac{a_k}{\sum_j a_j} \in \Delta_K.
\]
Then
\[
\log \sum_k a_k = \log \!\left( \sum_k q_k \cdot \frac{a_k}{q_k} \right).
\]
Apply Jensen’s inequality to the concave logarithm function:
\[
\log \!\left( \sum_k q_k \cdot \frac{a_k}{q_k} \right) \geq \sum_k q_k \log \left( \frac{a_k}{q_k} \right),
\]
with equality if and only if \( q_k \propto a_k \), i.e., \( q = p \). Now rewrite:
\begin{align*}
\log \sum_k a_k
&= \sum_k q_k \log \left( \frac{a_k}{q_k} \right)
= \sum_k q_k \left[ \log \!\left( \frac{a_k}{p_k} \cdot \frac{p_k}{q_k} \right) \right] \\
&= \sum_k q_k \left[ \log \!\left( \frac{a_k}{p_k} \right) + \log \!\left( \frac{p_k}{q_k} \right) \right]
= \sum_k q_k \log \!\left( \frac{a_k}{p_k} \right) - \sum_k q_k \log \!\left( \frac{q_k}{p_k} \right) \\
&= \sum_k q_k \log \!\left( \frac{a_k}{p_k} \right) + D_{\mathrm{KL}}(q \,\|\, p).
\end{align*}
But \( \dfrac{a_k}{p_k} = \sum_j a_j \), so the first term becomes
\[
\sum_k q_k \log \!\left( \sum_j a_j \right) = \log \!\left( \sum_j a_j \right),
\]
since \( \sum_k q_k = 1 \). Thus,
\[
\log \sum_k a_k = \sum_k q_k \log \left( \frac{a_k}{q_k} \right) + D_{\mathrm{KL}}(q \,\|\, p),
\]
which completes the proof.
\end{proof}

\begin{proof}[Proof of Theorem~\ref{theorem:ELBO for the MBT Model}]
We aim to derive a variational lower bound (ELBO) on the log-likelihood of observing a preference \( y^+ \succ y^- \) given a prompt \( x \), under the Mixture-of-Bradley--Terry (MBT) model.

The marginal likelihood under the MBT model is
\[
\mathbb{P}(y^+ \succ y^- \mid x) = \sum_{k=1}^K w_k(x)\, \sigma_k(x, y^+, y^-),
\]
where \( w_k(x) = p(z = k \mid x) \) is the prior mixture weight, and the expert-specific Bradley--Terry likelihood is
\[
\sigma_k(x, y^+, y^-) := \frac{\exp(r_k(x, y^+))}{\exp(r_k(x, y^+)) + \exp(r_k(x, y^-))}.
\]

Let \( q(z \mid x, y^+, y^-) \) be any variational distribution over the latent expert index \( z \in \{1, \dots, K\} \), and let \( q_k(x, y^+, y^-) := q(z = k \mid x, y^+, y^-) \). Then
\begin{align*}
\log \mathbb{P}(y^+ \succ y^- \mid x)
&= \log \sum_{k=1}^K w_k(x)\, \sigma_k(x, y^+, y^-) \\
&= \log \sum_{k=1}^K q_k(x, y^+, y^-) \cdot \frac{w_k(x)\, \sigma_k(x, y^+, y^-)}{q_k(x, y^+, y^-)} \\
&\geq \sum_{k=1}^K q_k(x, y^+, y^-) \log \left( \frac{w_k(x)\, \sigma_k(x, y^+, y^-)}{q_k(x, y^+, y^-)} \right),
\end{align*}
where the inequality follows from Jensen’s inequality applied to the concave logarithm function.

This is the standard form of the evidence lower bound (ELBO), and it can be rewritten as
\[
\sum_{k=1}^K q_k(x, y^+, y^-) \log \sigma_k(x, y^+, y^-) \;-\; D_{\mathrm{KL}}\!\left(q(\cdot \mid x, y^+, y^-) \,\|\, p(\cdot \mid x)\right).
\]

Thus, we obtain the lower bound
\[
\log \mathbb{P}(y^+ \succ y^- \mid x)
\geq \mathbb{E}_{z \sim q(\cdot \mid x, y^+, y^-)} \!\left[ \log \sigma_z(x, y^+, y^-) \right]
- D_{\mathrm{KL}}\!\left(q(z \mid x, y^+, y^-) \,\|\, p(z \mid x)\right).
\]

\noindent\textbf{Tightness of the bound.} The inequality becomes an equality when the variational posterior \( q(z \mid x, y^+, y^-) \) matches the true posterior \( p(z \mid x, y^+, y^-) \). Using Bayes’ rule, the exact posterior under the MBT model is
\begin{align}
p(z = k \mid x, y^+, y^-)
&= \frac{p(z = k \mid x)\, \mathbb{P}(y^+ \succ y^- \mid x, z = k)}{\mathbb{P}(y^+ \succ y^- \mid x)} \notag \\
&= \frac{w_k(x)\, \sigma_k(x, y^+, y^-)}{\sum_{j=1}^K w_j(x)\, \sigma_j(x, y^+, y^-)}. \notag
\end{align}
Therefore, the bound is tight when
\[
q_k(x, y^+, y^-) = \frac{w_k(x)\, \sigma_k(x, y^+, y^-)}{\sum_{j=1}^K w_j(x)\, \sigma_j(x, y^+, y^-)}.
\]
\end{proof}

\begin{proof}[Proof of Theorem~\ref{theorem:Equality Decomposition of Reward Mixture}]
Recall the definition of the mixture reward function from~\eqref{eq:mix_exp_reward}:
\[
r(x, y) := \log \sum_{k=1}^K w_k(x)\, \exp(r_k(x, y)).
\]
We apply Lemma~\ref{lemma:variationalIdentity} using \( a_k = w_k(x)\, \exp(r_k(x, y)) \) and setting \( q_k = q^{(\pi)}_k(x, y) \). This gives
\begin{align}
r(x, y)
&= \sum_{k=1}^K q^{(\pi)}_k(x, y) \left[ r_k(x, y) + \log w_k(x) - \log q^{(\pi)}_k(x, y) \right] \notag \\
&\quad + \sum_{k=1}^K q^{(\pi)}_k(x, y) \log \left( \frac{q^{(\pi)}_k(x, y)}{q^{(r)}_k(x, y)} \right)
\quad \left(\text{recall } q^{(r)}_k(x, y) = \frac{w_k(x)\, \exp(r_k(x, y))}{\sum_j w_j(x)\, \exp(r_j(x, y))}\right) \notag \\
&= \sum_{k=1}^K q^{(\pi)}_k(x, y) \left[ r_k(x, y) + \log w_k(x) - \log q^{(\pi)}_k(x, y) \right]
+ D_{\mathrm{KL}}\!\left(q^{(\pi)} \,\|\, q^{(r)} \right), \label{eq:reward_decomposition_intermediate}
\end{align}
which matches the decomposition in~\eqref{eq:reward_decomposition}. Therefore,
\[
r(x, y)
= \mathbb{E}_{z \sim q^{(\pi)}(\cdot \mid x, y)} \!\left[ r_z(x, y) + \log w_z(x) - \log q^{(\pi)}_z(x, y) \right]
+ D_{\mathrm{KL}}\!\left(q^{(\pi)} \,\|\, q^{(r)} \right).
\]
\end{proof}

\begin{proof}[Proof of Lemma~\ref{lemma:MoE-DPO Objective Decomposition}]
We begin from the MoE-DPO objective in~\eqref{eq:DPO_Training_Objective}:
\[
\mathcal{L}_{\text{MoE-DPO}} := \mathbb{E}_{x \sim p,\, y \sim \pi(y \mid x)} \!\left[ r(x, y) \right]
- \beta \sum_{k=1}^K \mathbb{E}_{x \sim p} \!\left[ w_k(x)\, D_{\mathrm{KL}} \!\left( \pi_k(y \mid x) \,\|\, \pi_{\text{ref}(k)}(y \mid x) \right) \right],
\]
where \( \pi(y \mid x) = \sum_{k=1}^K w_k(x)\, \pi_k(y \mid x) \) and the soft-mixture reward is defined by~\eqref{eq:mix_exp_reward}:
\[
r(x, y) := \log \sum_{k=1}^K w_k(x)\, \exp(r_k(x, y)).
\]
Apply Theorem~\ref{theorem:Equality Decomposition of Reward Mixture} to obtain
\begin{align}
r(x, y)
&= \sum_{k=1}^K q_k^{(\pi)}(x, y) \left[ r_k(x, y) + \log w_k(x) - \log q_k^{(\pi)}(x, y) \right]
+ D_{\mathrm{KL}}\!\left( q^{(\pi)} \,\|\, q^{(r)} \right), \label{eq:reward_decomp_applied}
\end{align}
where
\[
q_k^{(\pi)}(x, y) = \frac{w_k(x)\, \pi_k(y \mid x)}{\pi(y \mid x)}, \quad \text{and} \quad \sum_k q_k^{(\pi)}(x, y) = 1.
\]

Taking expectation over \( x \sim p \), \( y \sim \pi(y \mid x) \) gives
\begin{align*}
\mathbb{E}_{x\sim p,\, y \sim \pi(y \mid x)}[r(x, y)]
&= \sum_{k=1}^K \mathbb{E}_{x\sim p,\, y \sim \pi(y \mid x)} \!\left[ q_k^{(\pi)}(x, y) \left( r_k(x, y) + \log w_k(x) - \log q_k^{(\pi)}(x, y) \right) \right] \\
&\quad + \mathbb{E}_{x\sim p,\, y\sim \pi(y \mid x)} \!\left[ D_{\mathrm{KL}}(q^{(\pi)} \,\|\, q^{(r)}) \right].
\end{align*}
Hence,
\[
\mathbb{E}_{x\sim p,\, y \sim \pi(y \mid x)}[r(x, y)]
= \sum_{k=1}^K \mathbb{E}_{x\sim p,\, y \sim \pi(y \mid x)}
\!\left[ q_k^{(\pi)}(x, y) \left( r_k(x, y) + \log w_k(x) - \log q_k^{(r)}(x, y) \right) \right].
\]
This motivates defining the corrected reward
\[
\widetilde{r}_k(x, y) := r_k(x, y) - \log \left( \frac{q_k^{(r)}(x, y)}{w_k(x)} \right).
\]

Now use the identity \( q_k^{(\pi)}(x, y)\, \pi(y \mid x) = w_k(x)\, \pi_k(y \mid x) \), which implies
\[
\mathbb{E}_{x\sim p,\, y \sim \pi(y \mid x)} \!\left[ q_k^{(\pi)}(x, y)\, f_k(x, y) \right]
= \mathbb{E}_{x \sim p,\, y \sim \pi_k(y \mid x)} \!\left[ w_k(x)\, f_k(x, y) \right]
\]
for any measurable function \( f_k \). Applying this to \( \widetilde{r}_k(x, y) \) gives
\[
\mathbb{E}_{x\sim p,\, y \sim \pi(y \mid x)}[r(x, y)]
= \sum_{k=1}^K \mathbb{E}_{x,\, y \sim \pi_k(y \mid x)} \!\left[ w_k(x)\, \widetilde{r}_k(x, y) \right].
\]
Combining with the expert-specific KL penalties in the original objective, we conclude
\[
\mathcal{L}_{\text{MoE-DPO}} = \sum_{k=1}^K \mathbb{E}_{x \sim p,\, y \sim \pi_k(y \mid x)} \!\left[ w_k(x) \left( \widetilde{r}_k(x, y) - \beta \log \frac{\pi_k(y \mid x)}{\pi_{\text{ref}(k)}(y \mid x)} \right) \right],
\]
which completes the proof.
\end{proof}

\begin{proof}[Proof of Theorem~\ref{lemma:Optimal Expert Policy under KL-Regularized Objective}]
We seek the form of the expert policy \( \pi_k(y \mid x) \) that maximizes the per-expert objective:
\[
\mathcal{L}_k(\pi_k) = \mathbb{E}_{x \sim p} \!\left[
\mathbb{E}_{y \sim \pi_k(y \mid x)} \!\left[ \widetilde{r}_k(x, y) \right]
- \beta\, D_{\mathrm{KL}}\!\left( \pi_k(\cdot \mid x) \,\|\, \pi_{\mathrm{ref}(k)}(\cdot \mid x) \right)
\right],
\]
where \( \widetilde{r}_k(x, y) \) is the adjusted reward, and \( \pi_{\mathrm{ref}(k)}(y \mid x) \) is a fixed reference policy.

This objective is additive over \( x \), so it suffices to maximize the inner functional pointwise for each \( x \). Fix \( x \in \mathcal{X} \) and define
\[
\pi(y) := \pi_k(y \mid x), \quad
\pi_{\mathrm{ref}}(y) := \pi_{\mathrm{ref}(k)}(y \mid x), \quad
\widetilde{r}(y) := \widetilde{r}_k(x, y).
\]
We now solve the constrained optimization problem
\[
\max_{\pi \in \Delta_{\mathcal{Y}}}
\left\{
\sum_{y \in \mathcal{Y}} \pi(y)\, \widetilde{r}(y)
- \beta \sum_{y \in \mathcal{Y}} \pi(y) \log \frac{\pi(y)}{\pi_{\mathrm{ref}}(y)}
\right\},
\]
subject to \( \sum_{y} \pi(y) = 1 \), where \( \Delta_{\mathcal{Y}} \) denotes the probability simplex over the discrete output space \( \mathcal{Y} \).

Introduce a Lagrange multiplier \( \lambda \in \mathbb{R} \) for the normalization constraint and define the Lagrangian
\[
\mathcal{L}(\pi, \lambda) =
\sum_y \pi(y)\, \widetilde{r}(y)
- \beta \sum_y \pi(y) \log \frac{\pi(y)}{\pi_{\mathrm{ref}}(y)}
- \lambda \left( \sum_y \pi(y) - 1 \right).
\]
Taking derivatives with respect to \( \pi(y) \) gives
\[
\frac{\partial \mathcal{L}}{\partial \pi(y)}
= \widetilde{r}(y)
- \beta \left( \log \frac{\pi(y)}{\pi_{\mathrm{ref}}(y)} + 1 \right)
- \lambda.
\]
Setting this to zero yields
\[
\pi(y) = \pi_{\mathrm{ref}}(y)\, \exp\!\left( \frac{1}{\beta} \left( \widetilde{r}(y) - \lambda - \beta \right) \right).
\]
Let \( Z \) be the normalizing constant ensuring \( \sum_y \pi(y) = 1 \). Then
\[
\pi(y) = \frac{1}{Z}\, \pi_{\mathrm{ref}}(y)\, \exp\!\left( \frac{1}{\beta} \widetilde{r}(y) \right),
\quad
Z := \sum_{y'} \pi_{\mathrm{ref}}(y')\, \exp\!\left( \frac{1}{\beta} \widetilde{r}(y') \right).
\]
Thus, the optimal expert policy is
\[
\pi_k^*(y \mid x) = \frac{1}{Z_k(x)}\, \pi_{\mathrm{ref}(k)}(y \mid x)\, \exp\!\left( \frac{1}{\beta} \widetilde{r}_k(x, y) \right),
\]
where \( Z_k(x) = \sum_{y'} \pi_{\mathrm{ref}(k)}(y' \mid x)\, \exp\!\left( \frac{1}{\beta} \widetilde{r}_k(x, y') \right) \). The solution is unique because the objective is strictly concave in \( \pi \) (linear term minus KL), and the feasible set \( \Delta_{\mathcal{Y}} \) is convex and compact.
\end{proof}

\begin{proof}[Proof that maximizing the ELBO with respect to \( w_k(x) \) is equivalent to minimizing~\eqref{eq:wk_loss}]
We show that maximizing the ELBO with respect to the gating weights \( w_k(x) \) is equivalent to minimizing the expected KL divergence between the variational posterior and the gating prior, as given in~\eqref{eq:wk_loss}.

Recall the ELBO from Theorem~\ref{theorem:ELBO for the MBT Model}, which for a single preference triplet \( (x, y^+, y^-) \) is
\[
\log \mathbb{P}(y^+ \succ y^- \mid x)
\geq \sum_{k=1}^K q_k(x, y^+, y^-) \log \left( \frac{w_k(x)\, \sigma_k(x, y^+, y^-)}{q_k(x, y^+, y^-)} \right),
\]
where \( \sigma_k(x, y^+, y^-) \) is the Bradley--Terry likelihood under expert \( k \), and \( q_k(x, y^+, y^-) \) is any variational posterior over the expert index \( z \).

Isolate the part of the ELBO that depends on \( w_k(x) \). Since \( \sigma_k(x, y^+, y^-) \) and \( q_k(x, y^+, y^-) \) are fixed (as outputs of the E-step), the only term depending on \( w_k(x) \) is
\[
\sum_{k=1}^K q_k(x, y^+, y^-) \log w_k(x).
\]
Thus, maximizing the ELBO over \( w_k(x) \), subject to \( \sum_{k=1}^K w_k(x) = 1 \), is equivalent to solving
\[
\max_{\{w_k(x)\}} \ \sum_{k=1}^K q_k(x, y^+, y^-) \log w_k(x)
\quad \text{subject to } \sum_k w_k(x) = 1,\ w_k(x) \geq 0.
\]
This is equivalent to minimizing
\[
D_{\mathrm{KL}}\!\left(q(\cdot \mid x, y^+, y^-) \,\|\, w(\cdot \mid x)\right)
= \sum_{k=1}^K q_k(x, y^+, y^-) \log \frac{q_k(x, y^+, y^-)}{w_k(x)},
\]
which attains its minimum when \( w_k(x) = q_k(x, y^+, y^-) \).

Taking the expectation over \( (x, y^+, y^-) \sim \mathcal{D} \) yields the training objective
\[
\argmin_{w_k(x)} \ \mathbb{E}_{(x, y^+, y^-) \sim \mathcal{D}}
\left[ \sum_{k=1}^K q_k(x, y^+, y^-) \log \frac{q_k(x, y^+, y^-)}{w_k(x)} \right],
\]
which is precisely~\eqref{eq:wk_loss}. Hence, maximizing the ELBO with respect to \( w_k(x) \) is equivalent to minimizing the expected KL divergence between the variational posterior and the gating prior.
\end{proof}

\section{Stochastic EM Algorithm for Mix- and MoE-DPO}
\label{appdx:em_algorithm}

This appendix details the stochastic Expectation–Maximization (EM) procedure used to train the Mix- and MoE-DPO models. Rather than operating on the full dataset, each iteration of the algorithm is applied to a randomly sampled minibatch, enabling scalable training on large preference datasets. The stochastic EM loop alternates between soft expert assignment (E-step) and parameter optimization (M-step), and comprises the following four stages:

\begin{enumerate}[leftmargin=*]
    \item \textbf{E-step:} Posterior-responsibility computation over latent experts using the MBT model.
    \item \textbf{M-step:} Policy update via log-normalized preference gradients.
    \item \textbf{Reward update:} Reward recomputation for posterior consistency.
    \item \textbf{M-step:} Mixture-prior update using minibatch-responsibility averages.
\end{enumerate}

Each step is performed on a minibatch \( \{(x_i, y_i^+, y_i^-)\}_{i=1}^n \) of size \( n \), where the preference triplets are drawn from the full training dataset \( \mathcal{D} \).

\subsection{E-Step: Posterior-Responsibility Computation for Preference Triplets}
\label{appdx:e_step_mbt}

We compute the posterior responsibilities \( q_k(x, y^+, y^-) \) over the latent expert index \( z \in \{1, \ldots, K\} \) for each triplet in the minibatch. These responsibilities are derived from the expert-specific Bradley--Terry likelihoods and the current mixture weights:

\begin{algorithm}[H]
\caption{Posterior-Responsibility Computation for the MBT Model}
\label{alg:E_step_MBT_simple}
\begin{algorithmic}[1]
\Require Rewards \(\{r_k(x, y^+), r_k(x, y^-)\}_{k=1}^K\), mixture weights \(\{w_k(x)\}_{k=1}^K\).
\Ensure Posterior probabilities \(\{q_k(x, y^+, y^-)\}_{k=1}^K\).
\For{\(k = 1, \dots, K\)}
    \State Compute preference likelihood:
    \[
    p_k \gets \frac{\exp(r_k(x, y^+))}{\exp(r_k(x, y^+)) + \exp(r_k(x, y^-))}.
    \]
    \State Compute unnormalized responsibility:
    \[
    \tilde{q}_k \gets w_k(x) \cdot p_k.
    \]
\EndFor
\State Normalize responsibilities:
\[
Z \gets \sum_{j=1}^K \tilde{q}_j, \quad q_k(x, y^+, y^-) \gets \frac{\tilde{q}_k}{Z}.
\]
\Return \( \{q_k(x, y^+, y^-)\}_{k=1}^K \).
\end{algorithmic}
\end{algorithm}

\subsection{M-Step: Policy Update via Log-Space Normalized Gradient}
\label{appdx:mstep_policy_update}

Expert policies \( \pi_k \) are updated by maximizing the componentwise MBT objective using stochastic gradients. The loss for expert \( k \) is derived from the log-normalized likelihood ratio:
\[
\mathcal{L}_k = - \mathbb{E}_{(x, y^+, y^-) \sim \mathcal{D}} \left[
q_k(x, y^+, y^-) \cdot \log \left( \frac{A^+}{A^+ + A^-} \right)
\right],
\]
where we define
\begin{align*}
A^+ &= \pi_k(y^+ \mid x)^\beta \cdot q_k^{(r)}(x, y^+) \cdot \pi_{\text{ref}(k)}(y^+ \mid x)^{-\beta}, \\
A^- &= \pi_k(y^- \mid x)^\beta \cdot q_k^{(r)}(x, y^-) \cdot \pi_{\text{ref}(k)}(y^- \mid x)^{-\beta},
\end{align*}
and \( \beta > 0 \) is the inverse temperature.

\textbf{Gradient derivation.} We differentiate \( \mathcal{L}_k \) with respect to \( \theta_k \), the parameters of \( \pi_k \):
\begin{align*}
\nabla_{\theta_k} \mathcal{L}_k
&= - \mathbb{E}_{(x, y^+, y^-)} \left[
q_k(x, y^+, y^-) \cdot \nabla_{\theta_k} \log \left( \frac{A^+}{A^+ + A^-} \right)
\right] \\
&= - \mathbb{E}_{(x, y^+, y^-)} \left[
q_k(x, y^+, y^-) \cdot \left( \frac{1}{A^+} \nabla_{\theta_k} A^+ - \frac{1}{A^+ + A^-} \left( \nabla_{\theta_k} A^+ + \nabla_{\theta_k} A^- \right) \right)
\right].
\end{align*}
Using
\[
\nabla_{\theta_k} A^+ = \beta A^+ \nabla_{\theta_k} \log \pi_k(y^+ \mid x), \qquad
\nabla_{\theta_k} A^- = \beta A^- \nabla_{\theta_k} \log \pi_k(y^- \mid x),
\]
we substitute into the gradient:
\begin{align*}
\nabla_{\theta_k} \mathcal{L}_k
&= -\beta \, \mathbb{E}_{(x, y^+, y^-)} [
q_k(x, y^+, y^-)
\cdot (
\nabla_{\theta_k} \log \pi_k(y^+ \mid x)
- \frac{A^+}{A^+ + A^-} \nabla_{\theta_k} \log \pi_k(y^+ \mid x)
\nonumber\\&- \frac{A^-}{A^+ + A^-} \nabla_{\theta_k} \log \pi_k(y^- \mid x)
)
] \\
&= \beta \, \mathbb{E}_{(x, y^+, y^-)} \left[
q_k(x, y^+, y^-)
\cdot \left(
\frac{A^+}{A^+ + A^-} \nabla_{\theta_k} \log \pi_k(y^- \mid x)
- \frac{A^-}{A^+ + A^-} \nabla_{\theta_k} \log \pi_k(y^+ \mid x)
\right)
\right].
\end{align*}

\textbf{Stochastic gradient estimator.} On a minibatch \( \{(x_i, y_i^+, y_i^-)\}_{i=1}^n \), the gradient is approximated as
\begin{align}
\widehat{\nabla}_{\theta_k} \mathcal{L}_k^{\text{MBT}} =
\frac{\beta}{n} \sum_{i=1}^n q_k(x_i, y_i^+, y_i^-)
\left[
\frac{A_i^+}{A_i^+ + A_i^-} \nabla_{\theta_k} \log \pi_k(y_i^- \mid x_i)
- \frac{A_i^-}{A_i^+ + A_i^-} \nabla_{\theta_k} \log \pi_k(y_i^+ \mid x_i)
\right], \nonumber
\end{align}
with
\begin{align*}
A_i^+ &= \pi_k(y_i^+ \mid x_i)^\beta \cdot q_k^{(r)}(x_i, y_i^+) \cdot \pi_{\text{ref}(k)}(y_i^+ \mid x_i)^{-\beta}, \\
A_i^- &= \pi_k(y_i^- \mid x_i)^\beta \cdot q_k^{(r)}(x_i, y_i^-) \cdot \pi_{\text{ref}(k)}(y_i^- \mid x_i)^{-\beta}.
\end{align*}
This gradient is used to update \( \theta_k \) via standard optimizers such as SGD or Adam.

\subsection{Reward Update: Reward Recalibration for Posterior Consistency}
\label{appdx:reward_update}

Rewards \( r_k(x, y) \) are updated to enforce agreement between the reward-induced and policy-induced posteriors, using only the current minibatch:

\begin{algorithm}[H]
\caption{Reward Update for Posterior Reweighting}
\label{alg:reward_update_plain}
\begin{algorithmic}[1]
\Require Current policy \(\pi_k(y \mid x)\), reference policy \(\pi_{\text{ref}(k)}(y \mid x)\), responsibilities \(q_k^{(r)}(x, y)\), mixture weights \(w_k(x)\), temperature \(\beta\).
\Ensure Updated reward values \(r_k(x_i, y_i)\) for \(y_i \in \{y_i^+, y_i^-\}\).
\For{each expert \(k = 1, \dots, K\)}
    \For{each input \(x_i\) in the minibatch}
        \State Compute the partition function:
        \[
        Z_k(x_i) = \sum_{j=1}^n \pi_{\text{ref}(k)}(y_j \mid x_i)
        \exp\!\left( \frac{1}{\beta} \left( r_k(x_i, y_j) + \log w_k(x_i) - \log q_k^{(r)}(x_i, y_j) \right) \right).
        \]
        \For{each \(y_i \in \{y_i^+, y_i^-\}\)}
            \State Update the reward:
            \[
            r_k(x_i, y_i) =
            \beta \cdot \log\!\left( \frac{ \pi_k(y_i \mid x_i) \, Z_k(x_i) }{ \pi_{\text{ref}(k)}(y_i \mid x_i) } \right)
            + \log\!\left( \frac{ q_k^{(r)}(x_i, y_i) }{ w_k(x_i) } \right).
            \]
        \EndFor
    \EndFor
\EndFor
\end{algorithmic}
\end{algorithm}

\subsection{M-Step: Mixture-Prior Update (Batch-wise)}
\label{appdx:prior_update}

Mixture priors \( w_k(x) \) are updated using the minibatch-specific posterior responsibilities. The update differs between Mix-DPO and MoE-DPO.

\textbf{Mix-DPO (input-independent weights).}  
The mixture weights \( w_k \in \Delta^{K-1} \) are global parameters. The minibatch update is
\[
w_k \leftarrow \frac{1}{n} \sum_{i=1}^n q_k(x_i, y_i^+, y_i^-),
\]
where \( n \) is the minibatch size. This estimate replaces the full-batch average and reflects the current soft-assignment statistics over the minibatch.

\textbf{MoE-DPO (input-dependent gating).}  
The gating function \( w_k(x) = \text{softmax}(g_k(x)) \) is trained by minimizing the minibatch-averaged KL divergence:
\[
\min_{w_k(x)} \ \frac{1}{n} \sum_{i=1}^n \sum_{k=1}^K q_k(x_i, y_i^+, y_i^-) \log \frac{q_k(x_i, y_i^+, y_i^-)}{w_k(x_i)}.
\]
This corresponds to supervised learning over inputs \( x_i \) with soft labels \( \{q_k(x_i, y_i^+, y_i^-)\} \), implemented via cross-entropy loss on the gating logits.

\subsection{Numerical Considerations}
\label{appdx:numerical_notes}

All computations are performed on minibatches to support scalable training. We apply the log-sum-exp trick to evaluate partition functions and normalizers for numerical stability. Gradients are computed via automatic differentiation and optimized with standard stochastic optimizers.

\section{Mix- and MoE-DPO at Scale}\label{sec:extended_mix_moe_dpo}

The Mix- and MoE-DPO framework supports a range of scalable training regimes, offering flexibility across modeling capacity, inference fidelity, and computational constraints. Depending on the deployment setting, different components of the model can be emphasized or fixed during optimization. For instance, when expert policies \( \{\pi_k\}_{k=1}^K \) are initialized from strong zero-shot models (e.g., domain-specialized LLMs), they can be held fixed while rewards are computed using the canonical DPO transformation. In such cases, only the gating network, which outputs mixture weights \( \{w_k(x)\}_{k=1}^K \), is trained—including for a personalized gating function with user-specific characteristics \( u \in \mathcal{U} \). This reduces the procedure to updating the gating function via the mixture M-step of a variational EM algorithm, with the E-step used to infer expert-assignment posteriors.

In contrast, the fully end-to-end optimization strategy updates the expert policies, gating network, and reward functions jointly. To support both modular and end-to-end regimes, we introduce two complementary estimation strategies. First, a \textit{regularized variational inference} method imposes entropy- and KL-based penalties on the expert-assignment posterior, encouraging confident yet diverse expert selection while preserving a closed-form expression for the variational distribution. Second, a \textit{Monte Carlo relaxation} leverages the Gumbel–Softmax reparameterization trick~\cite{jang2017categorical} to enable differentiable sampling over the latent expert index. Both approaches support efficient training with minibatches and eliminate the need for explicit (full-sample) E-step computations in classical EM. The variational strategy stabilizes the E-step by enforcing structure on the posterior, while the Monte Carlo relaxation bypasses it entirely by enabling gradient-based updates through sampled expert assignments. Practitioners can instantiate expert policies as either fixed modules or fully trainable heads, depending on the complexity of the problem, computational resources, and adaptability requirements. Our estimation strategies are compatible with large-scale minibatch training and can scale to billions of parameters without incurring dense-model costs.

The framework supports sparse activation, modular reuse, and personalized alignment, making it suitable for deployment in multi-task, multi-user, and resource-constrained settings.

\subsection{Regularized Variational Responsibilities for the MBT Model}

We extend the Mix- and MoE-DPO framework by introducing a richer regularization structure specifically targeting the MBT posterior distribution \( q_k(x, y^+, y^-) \). The modified regularized MBT loss is given by:
\[
\begin{aligned}
\mathcal{L}_{\mathrm{MBT}}^{\mathrm{reg}}(\{q_k\}, \{\pi_k\}, \{w_k\})
&= \sum_{k=1}^K \mathbb{E}_{(x, y^+, y^-)\sim\mathcal{D}}\Bigg[q_k(x, y^+, y^-) \,\ell_k(x, y^+, y^-)
\\
&\quad + \lambda_{\mathrm{ent}}\,\mathcal{H}\!\left(q(\cdot \mid x, y^+, y^-)\right)
\\
&\quad - \lambda_{\mathrm{conf}}\,\mathcal{H}\!\left(q(\cdot \mid x, y^+, y^-)\right)
\\
&\quad - \lambda_{\mathrm{KL\_unif}}\,D_{\mathrm{KL}}\!\left(q(\cdot \mid x, y^+, y^-) \,\|\, \mathrm{Uniform}(\{1,\dots,K\})\right)
\\
&\quad - \lambda_{\mathrm{KL\_w}}\,D_{\mathrm{KL}}\!\left(q(\cdot \mid x, y^+, y^-) \,\|\, w(\cdot \mid x)\right)
\\
&\quad - \lambda_{\mathrm{KL\_w\_global}}\,\mathbb{E}_{x\sim p(x)}\!\left[D_{\mathrm{KL}}\!\left(w(\cdot \mid x)\,\|\,\mathrm{Uniform}(\{1,\dots,K\})\right)\right]\Bigg],
\end{aligned}
\]
where the per-component utility is
\[
\ell_k(x, y^+, y^-):=
\log \frac{
\left(\dfrac{\pi_k(y^+\mid x)}{\pi_{\mathrm{ref}(k)}(y^+\mid x)}\right)^\beta \, q_k^{(r)}(x, y^+)
}{
\sum_{\eta\in\{y^+, y^-\}}\left(\dfrac{\pi_k(\eta\mid x)}{\pi_{\mathrm{ref}(k)}(\eta\mid x)}\right)^\beta \, q_k^{(r)}(x, \eta)
},
\]
and \( \beta, \lambda_{\mathrm{ent}}, \lambda_{\mathrm{conf}}, \lambda_{\mathrm{KL\_unif}}, \lambda_{\mathrm{KL\_w}}, \lambda_{\mathrm{KL\_w\_global}} \) control the strength of the respective regularizers.

Each regularizer acts specifically on \( q_k(x, y^+, y^-) \) with distinct roles at various training phases:
\begin{itemize}
    \item \textbf{\( \beta \)}: Controls deviation from reference policies \( \pi_{\mathrm{ref}(k)} \). Active throughout, typically decaying slowly.
    \item \textbf{\( \lambda_{\mathrm{ent}} \)}: Promotes exploration via high entropy in responsibilities \( q_k(x, y^+, y^-) \). Critical early in training.
    \item \textbf{\( \lambda_{\mathrm{conf}} \)}: Drives specialization via low entropy, enforcing confident expert assignments. Crucial in later training phases.
    \item \textbf{\( \lambda_{\mathrm{KL\_unif}} \)}: Prevents early collapse onto few experts, regularizing toward uniformity. Emphasized early.
    \item \textbf{\( \lambda_{\mathrm{KL\_w}} \)}: Aligns posterior responsibilities \( q_k(x, y^+, y^-) \) with prior mixture weights \( w_k(x) \). Useful when priors are informative.
    \item \textbf{\( \lambda_{\mathrm{KL\_w\_global}} \)}: Regularizes global mixture weights toward uniformity. Strong early, relaxed later.
\end{itemize}

\noindent Note that the extended MBT objective introduces additional regularization terms on the variational posterior, which modify the computation of \( q_k(x, y^+, y^-) \) for each expert. Namely,
\begin{lemma}[Optimal Variational Posterior under Regularized MBT Objective]
\label{lemma:optimal_variational_posterior_MBT}
Assuming fixed expert policies \( \{\pi_k\} \), mixture weights \( \{w_k(x)\} \), and MBT reward structure \( \{q_k^{(r)}(x, y)\} \), the optimal variational posterior \( q_k(x, y^+, y^-) \) that maximizes the regularized MBT objective
\begin{align*}
\mathcal{L}^{\text{reg}}_{\text{MBT}}(q) &= \sum_{k=1}^K \mathbb{E}_{(x, y^+, y^-)\sim\mathcal{D}}\biggl[ q_k(x, y^+, y^-)
\log \frac{
\left(\dfrac{\pi_k(y^+\mid x)}{\pi_{\mathrm{ref}(k)}(y^+\mid x)}\right)^\beta q_k^{(r)}(x, y^+)
}{
\sum_{\eta\in\{y^+, y^-\}}\left(\dfrac{\pi_k(\eta\mid x)}{\pi_{\mathrm{ref}(k)}(\eta\mid x)}\right)^\beta q_k^{(r)}(x, \eta)
}
\\
&\quad + \lambda_{\mathrm{ent}}\,\mathcal{H}\!\left(q(\cdot \mid x, y^+, y^-)\right) - \lambda_{\mathrm{conf}}\,\mathcal{H}\!\left(q(\cdot \mid x, y^+, y^-)\right)
\\
&\quad - \lambda_{\mathrm{KL\_unif}}\,D_{\mathrm{KL}}\!\left(q(\cdot \mid x, y^+, y^-) \,\|\, \mathrm{Uniform}(\{1,\dots,K\})\right)
\\
&\quad - \lambda_{\mathrm{KL\_w}}\,D_{\mathrm{KL}}\!\left(q(\cdot \mid x, y^+, y^-) \,\|\, w(\cdot \mid x)\right) \biggr]
\end{align*}
is given by the normalized distribution
\[
q_k(x, y^+, y^-) = \frac{1}{Z(x, y^+, y^-)} \left(w_k(x)\right)^{\lambda_{\mathrm{KL\_w}}/\alpha}
\exp\!\left(\frac{1}{\alpha}\log\frac{\left(\dfrac{\pi_k(y^+\mid x)}{\pi_{\mathrm{ref}(k)}(y^+\mid x)}\right)^\beta q_k^{(r)}(x, y^+)}{\sum_{\eta\in\{y^+, y^-\}}\left(\dfrac{\pi_k(\eta\mid x)}{\pi_{\mathrm{ref}(k)}(\eta\mid x)}\right)^\beta q_k^{(r)}(x, \eta)}\right),
\]
where the effective temperature is
\[
\alpha := \lambda_{\mathrm{conf}} - \lambda_{\mathrm{ent}} + \lambda_{\mathrm{KL\_unif}} + \lambda_{\mathrm{KL\_w}},
\]
and \( Z(x, y^+, y^-) \) ensures normalization.
\end{lemma}

\begin{proof}[Proof of Lemma~\ref{lemma:optimal_variational_posterior_MBT}]
We maximize the integrand of the regularized MBT objective at each fixed \( (x, y^+, y^-) \). Define the per-component utility as
\[
\ell_k(x, y^+, y^-) := \log\frac{\left(\dfrac{\pi_k(y^+\mid x)}{\pi_{\mathrm{ref}(k)}(y^+\mid x)}\right)^\beta q_k^{(r)}(x, y^+)}{\sum_{\eta\in\{y^+, y^-\}}\left(\dfrac{\pi_k(\eta\mid x)}{\pi_{\mathrm{ref}(k)}(\eta\mid x)}\right)^\beta q_k^{(r)}(x, \eta)}.
\]
Grouping entropy and KL-based regularization terms, introduce
\[
\alpha := \lambda_{\mathrm{conf}} - \lambda_{\mathrm{ent}} + \lambda_{\mathrm{KL\_unif}} + \lambda_{\mathrm{KL\_w}}.
\]
The simplified optimization objective becomes
\[
\mathcal{L}(q) = \sum_{k=1}^K q_k\Big[\ell_k + \lambda_{\mathrm{KL\_w}}\log w_k(x) - \lambda_{\mathrm{KL\_unif}}\log K - \alpha \log q_k\Big] + \text{const}.
\]
Introducing a Lagrange multiplier \( \lambda \) for the simplex constraint \( \sum_k q_k = 1 \), setting derivatives to zero, and exponentiating yields
\[
q_k(x, y^+, y^-) = \frac{1}{Z(x, y^+, y^-)} \left(w_k(x)\right)^{\lambda_{\mathrm{KL\_w}}/\alpha}
\exp\!\left(\frac{\ell_k(x, y^+, y^-)}{\alpha}\right).
\]
Normalization ensures \( \sum_k q_k(x, y^+, y^-) = 1 \), completing the proof.
\end{proof}

These responsibilities, utilized throughout the EM procedure, reflect a regularized optimization that balances reward alignment with entropy control and KL constraints to priors, specifically the uniform distribution and the mixture weights \( w_k(x) \). While the updates of the expert policies \( \pi_k \) and the mixture weights \( w_k(x) \) maintain their structural form—being explicit functions of these responsibilities—the responsibilities \( q_k(x, y^+, y^-) \) themselves must replace their non-regularized counterparts across all EM algorithm steps. Consequently, this revised formulation preserves the modular decomposition of EM updates but mandates the use of refined, regularized responsibilities in both the M-step policy updates and the estimation of mixture weights.

For the training strategy, we recommend that (i) only a subset of the regularizers is activated during specific training phases; and (ii) typically, only one entropy-based regularizer (\( \lambda_{\mathrm{ent}} \) or \( \lambda_{\mathrm{conf}} \)) is employed at any given time.

We suggest dynamically adapting the regularization scheme across training phases to achieve a balance between exploration, specialization, and stability:

\begin{center}
\renewcommand{\arraystretch}{1.3}
\begin{tabular}{|c|c|c|}
\hline
\textbf{Phase} & \textbf{Objective} & \textbf{Active Regularizers} \\
\hline
Early (Exploration) &
Maintain diversity, avoid expert collapse &
\( \lambda_{\mathrm{ent}} \), \( \lambda_{\mathrm{KL\_unif}} \), \( \lambda_{\mathrm{KL\_w\_global}} \) \\
\hline
Middle (Specialization) &
Encourage confident expert specialization &
\( \lambda_{\mathrm{conf}} \), moderate \( \beta \) \\
\hline
Late (Stabilization) &
Solidify learned specialization, maintain stability &
\( \lambda_{\mathrm{conf}} \), \( \lambda_{\mathrm{KL\_w}} \), low \( \beta \) \\
\hline
\end{tabular}
\end{center}

\noindent
During the early training phase, high-entropy and uniformity regularizers encourage broad expert participation. In the middle phase, confidence-promoting regularizers sharpen the variational responsibilities and encourage expert specialization. In the late phase, regularization terms are adjusted to maintain stability and finalize expert specialization based on learned structures.

\subsection{Monte Carlo Relaxation for Mixture of Bradley--Terry Estimation}

To eliminate the explicit E-step in the variational EM algorithm for the Mixture of Bradley--Terry (MBT) model, we adopt a Monte Carlo relaxation based on the Gumbel–Softmax reparameterization, which allows replacing hard assignments (e.g., selecting one expert via \(\arg\max\)) with soft, differentiable assignments during training. Instead of computing the closed-form posterior responsibilities \( q_k(x, y^+, y^-) \), we approximate the latent expert assignment \( z \in \{1, \dots, K\} \) by sampling from a continuous relaxation of the categorical distribution.

Let \( \alpha_k(x, y^+, y^-) \) denote the unnormalized expert logit scores,
\[
\alpha_k(x, y^+, y^-) \propto w_k(x) \cdot \frac{\exp(r_k(x, y^+))}{\exp(r_k(x, y^+)) + \exp(r_k(x, y^-))},
\]
and define \( \log \alpha(x, y^+, y^-) \in \mathbb{R}^K \) as the vector of logits across components. We draw \( L \) samples from the Gumbel–Softmax distribution with temperature \( \tau > 0 \),
\[
z^{(l)} \sim \mathrm{GumbelSoftmax}(\log \alpha(x, y^+, y^-), \tau), \quad l = 1, \ldots, L,
\]
i.e.,
\[
z_k = \frac{\exp\!\left((\log w_k(x) + g_k)/\tau\right)}{\sum_{j=1}^K \exp\!\left((\log w_j(x) + g_j)/\tau\right)}, \quad \text{where } g_k \sim \mathrm{Gumbel}(0, 1).
\]
Hence \( z^{(l)} \in \Delta^{K-1} \) lies in the continuous simplex, and these soft assignments act as differentiable surrogates for \( q_k(x, y^+, y^-) \). We replace the original MBT loss in \eqref{eq:MBT_Loss_def} with a Monte Carlo estimate of the relaxed objective:
\[
\mathcal{L}_{\mathrm{MC}} = -\frac{1}{L} \sum_{l=1}^L \sum_{k=1}^K z_k^{(l)} \log \left( \frac{\exp(r_k(x, y^+))}{\exp(r_k(x, y^+)) + \exp(r_k(x, y^-))} \right).
\]
Expressing the sample \( z \) as a deterministic, differentiable function of \( w_k(x) \) and auxiliary random noise \( g_k \) enables low-variance gradient estimation through the sampling step.

The KL divergence between the variational posterior and the prior \( \mathrm{KL}(q(z \mid x, y^+, y^-) \,\|\, p(z \mid x)) \) can also be approximated using relaxed samples:
\[
\widehat{\mathrm{KL}} = \frac{1}{L} \sum_{l=1}^L \sum_{k=1}^K z_k^{(l)} \log \left( \frac{z_k^{(l)}}{w_k(x)} \right).
\]
Finally, policy parameters \( \theta_k \), rewards \( r_k(x, y) \), and gating-network weights \( w_k(x) \) are updated using these relaxed estimates without requiring a closed-form E-step. This enables fully \textbf{end-to-end} training using backpropagation under the Mix- and MoE-DPO frameworks while preserving a variational lower bound on the marginal likelihood.

\textbf{Compatibility with General Architectures.} Monte Carlo relaxation extends seamlessly to architectures where the MoE structure is embedded within transformer layers or other neural architectures, without requiring explicit policy factorization. In such settings:
\begin{itemize}
    \item Mixture weights \( w_k(x) \) are learned via differentiable gating networks.
    \item Expert rewards \( r_k(x, y) \) can be computed via head activations or contrastive objectives.
    \item The relaxed latent variable \( z \) enables sparse, input-adaptive routing.
\end{itemize}
This flexibility makes Monte Carlo relaxation a method of choice in high-capacity deployments, allowing scalable and differentiable training across diverse preference tasks.

\section{Convergence of the Mix- and MoE-DPO Estimation}\label{sec:convergence_analysis}
We report sufficient conditions for the convergence of the Mix- and MoE-DPO estimation procedures under two algorithmically distinct training paradigms:

\begin{enumerate}
    \item A variational EM algorithm with closed-form expert responsibilities, alternating between inference over latent variables and maximization of a variational lower bound~\cite{robbins1951stochastic}.
    \item A Monte Carlo relaxation approach based on the Gumbel--Softmax trick~\cite{jang2017categorical, maddison2017concrete}, which enables fully differentiable, end-to-end optimization over latent expert assignments.
\end{enumerate}

While the two methods differ in their treatment of latent variables and optimization flow, they share a unified objective: maximization of a stochastic evidence lower bound (ELBO) under potentially noisy minibatch updates. To this end, we specify a set of shared assumptions on smoothness, boundedness, and variance control that guarantee convergence to stationary points of the ELBO, as detailed in Sections~\ref{appdx:em_convergence} and~\ref{appdx:mc_convergence}.

\subsection{Shared Assumptions and General Conditions}\label{appdx:general_conditions}

Let \( \theta \) denote the model parameters (e.g., expert policies \( \{\pi_k\} \), gating weights \( \{w_k(x)\} \)), and let \( \mathcal{L}(\theta) \) denote a (possibly relaxed) ELBO objective.

\vspace{1ex}
\noindent\textbf{General Assumptions for Convergence}
\begin{itemize}
    \item \textbf{(Smoothness)} The objective \( \mathcal{L}(\theta) \) is continuously differentiable in \( \theta \), and its gradient \( \nabla_\theta \mathcal{L}(\theta) \) is Lipschitz continuous on compact subsets.
    \item \textbf{(Boundedness)} The ELBO objective is finite over the admissible domain:
    \[
    -\infty < \mathcal{L}(\theta) < \infty \quad \text{for all } \theta \in \Theta.
    \]
    \item \textbf{(Coercivity)} The objective is coercive, i.e.,
    \[
    \|\theta\| \to \infty \quad \Rightarrow \quad \mathcal{L}(\theta) \to -\infty.
    \]
    \item \textbf{(Stochastic Approximation)} If gradient updates are computed via stochastic samples, we assume:
    \begin{itemize}
        \item Robbins--Monro step size conditions~\cite{robbins1951stochastic}: \( \sum_t \eta_t = \infty \), \( \sum_t \eta_t^2 < \infty \).
        \item Unbiased stochastic gradient estimates.
        \item Finite variance of the gradient noise:
        \[
        \mathbb{E}\big\| \widehat{\nabla}_\theta \mathcal{L}(\theta)- \nabla_\theta\mathcal{L}(\theta) \big\|^2 \leq C\big(1 + \|\theta\|^2\big).
        \]
    \end{itemize}
\end{itemize}

These assumptions are standard in stochastic variational inference and are generally satisfied in the Mix- and MoE-DPO framework under practical architectural and training choices.

\textbf{Smoothness} holds naturally in our model due to the use of softmax parameterizations for policies and gating networks, which are differentiable with Lipschitz-continuous gradients over compact parameter domains. When using Gumbel--Softmax relaxations for latent-variable inference, the objective remains smooth as long as the temperature \( \tau \) is kept strictly positive during training. Potential violations can occur if gating functions degenerate into hard assignments prematurely (\( \tau \to 0 \)), which may introduce sharp, non-smooth transitions.

\textbf{Boundedness} of the ELBO is preserved by the finite output space of softmax policies and regularization through KL terms. All terms in the ELBO remain bounded as long as the expert policies do not collapse to delta distributions disjoint from their reference policies. This risk is mitigated through KL regularization and initialization from well-calibrated base models. In practice, numerical instability is avoided by ensuring overlap in support between \( \pi_k \) and \( \pi_{\text{ref}(k)} \), especially in early training.

\textbf{Coercivity} is enforced by the KL-divergence regularization, which grows superlinearly as the expert policy deviates from its reference. This ensures that unbounded parameter growth leads to penalization in the ELBO. However, coercivity may be weakened if the gating network assigns vanishing weights \( w_k(x) \approx 0 \) to specific experts, effectively nullifying their KL terms. This emphasizes the importance of entropy control in the gating mechanism to prevent expert collapse or neglect.

\textbf{Stochastic Approximation} assumptions are supported by minibatch-based training, uniform data sampling, and unbiased reparameterization gradients in the Gumbel--Softmax case. The Robbins--Monro conditions on learning-rate decay are satisfied via standard schedules. Finite gradient variance is generally ensured by using moderate temperatures (\( \tau > 0.5 \)), sufficient samples per minibatch, and normalized reward magnitudes. Variance may increase if reward values are poorly scaled or if the number of latent samples \( L \) is too small in the Monte Carlo setting.

In summary, these assumptions are valid under standard training setups for Mix- and MoE-DPO. Caution is warranted in highly sparse, low-temperature, or degenerate expert-routing regimes, which may require annealing strategies, sample averaging, or additional regularization to maintain convergence guarantees.

\subsection{Case 1: Variational EM Algorithm}\label{appdx:em_convergence}

In the variational EM setting, we consider a variational posterior \( q(z \mid x, y^+, y^-) \in \Delta^{K-1} \) and maximize the ELBO:
\[
\mathcal{L}(q, \theta) = \mathbb{E}_{q}\!\left[\log p(y^+, y^-, z \mid x, \theta)\right] - \mathbb{E}_{q}\!\left[\log q(z \mid x, y^+, y^-)\right].
\]
This objective is optimized alternately via:
\begin{itemize}
    \item \textbf{E-step:} Maximize \( \mathcal{L}(q, \theta^{(t)}) \) with respect to \( q \) while holding \( \theta^{(t)} \) fixed;
    \item \textbf{M-step:} Maximize \( \mathcal{L}(q^{(t+1)}, \theta) \) with respect to \( \theta \) while holding \( q^{(t+1)} \) fixed.
\end{itemize}

\noindent\textbf{Convergence Theorem (Variational EM).} Under the smoothness, boundedness, and coercivity assumptions defined in Section~\ref{appdx:general_conditions}, the sequence of variational EM updates satisfies
\[
\nabla_\theta \mathcal{L}(q^{(t)}, \theta^{(t)}) \to 0, \quad q^{(t)} \to q^*, \quad \theta^{(t)} \to \theta^*,
\]
i.e., the sequence converges to a stationary point of the ELBO objective. This follows from standard convergence results for coordinate ascent applied to variational EM~\cite{robbins1951stochastic}.

\begin{proof}
The Mix- and MoE-DPO model is trained by maximizing the marginal likelihood of preference comparisons under a latent-variable mixture model:
\[
\mathbb{P}(y^+ \succ y^- \mid x) = \sum_{k=1}^K w_k(x) \cdot \frac{\exp(r_k(x, y^+))}{\exp(r_k(x, y^+)) + \exp(r_k(x, y^-))},
\]
where the latent variable \( z \in \{1, \dots, K\} \) denotes the expert index, \( w_k(x) \) is the input-dependent prior over experts, and \( r_k(x, y) \) is the reward function for expert \( k \).

We consider the variational EM algorithm that maximizes the evidence lower bound (ELBO) for the marginal likelihood of preference observations. Each iteration \( t \) consists of:

\textbf{E-step:} Update the variational posterior \( q^{(t+1)}(z = k \mid x, y^+, y^-) \) using
\[
q^{(t+1)}_k(x, y^+, y^-) = \frac{w_k^{(t)}(x) \cdot \dfrac{\exp(r_k^{(t)}(x, y^+))}{\exp(r_k^{(t)}(x, y^+)) + \exp(r_k^{(t)}(x, y^-))}}{\sum_{j=1}^K w_j^{(t)}(x) \cdot \dfrac{\exp(r_j^{(t)}(x, y^+))}{\exp(r_j^{(t)}(x, y^+)) + \exp(r_j^{(t)}(x, y^-))}},
\]
which minimizes the KL divergence to the true posterior and hence increases the ELBO.

\textbf{M-step:} Perform stochastic gradient ascent on the ELBO:
\[
\theta^{(t+1)} = \theta^{(t)} + \eta_t \nabla_\theta \mathcal{L}(q^{(t+1)}, \theta^{(t)}),
\]
where \( \theta \) denotes the collection of model parameters, including \( \{\pi_k\} \), \( \{w_k\} \), and potentially shared parameters, and \( \eta_t \) satisfies the Robbins--Monro conditions \( \sum_t \eta_t = \infty \), \( \sum_t \eta_t^2 < \infty \).

\textbf{Step 1 (Monotonic improvement).}
The E-step yields a strict increase (or no decrease) in the ELBO due to Jensen’s inequality:
\[
\mathcal{L}(q^{(t+1)}, \theta^{(t)}) \geq \mathcal{L}(q^{(t)}, \theta^{(t)}).
\]
The M-step, via stochastic gradient ascent, increases \( \mathbb{E}[\mathcal{L}(q^{(t+1)}, \theta^{(t+1)})] \) in expectation.

\textbf{Step 2 (Boundedness).}
By assumption, \( \mathcal{L}(q, \theta) \) is bounded above. Hence, the sequence \( \{\mathcal{L}(q^{(t)}, \theta^{(t)})\}_{t=1}^\infty \) is monotone increasing and bounded, and therefore converges.

\textbf{Step 3 (Robbins--Monro convergence).}
Assuming bounded variance of stochastic gradients and that \( \nabla_\theta \mathcal{L} \) is Lipschitz, the standard Robbins--Monro theorem implies
\[
\nabla_\theta \mathcal{L}(q^{(t)}, \theta^{(t)}) \to 0 \quad \text{almost surely}.
\]

\textbf{Step 4 (Convergence to a stationary point).}
Since the ELBO is continuously differentiable in both arguments and \( \mathcal{L}(q^{(t)}, \theta^{(t)}) \) converges, we obtain
\[
q^{(t)} \to q^*, \quad \theta^{(t)} \to \theta^*, \quad \text{with} \quad \nabla_\theta \mathcal{L}(q^*, \theta^*) = 0.
\]
This implies that \( (q^*, \theta^*) \) is a stationary point of the ELBO objective.
\end{proof}

\subsection{Case 2: Monte Carlo Relaxation}\label{appdx:mc_convergence}

In this formulation, the latent responsibility \( z \in \{1, \dots, K\} \) is approximated using a soft Gumbel--Softmax sample \( z^{(l)} \in \Delta^{K-1} \). The training objective becomes
\[
\widehat{\mathcal{L}}(\theta) = \frac{1}{L} \sum_{l=1}^L \log p(y^+, y^-, z^{(l)} \mid x; \theta) - \log p(z^{(l)} \mid x),
\]
where \( z^{(l)} \sim \mathrm{GumbelSoftmax}(\log \alpha(x), \tau) \)~\cite{jang2017categorical, maddison2017concrete}.

\noindent\textbf{Convergence Theorem (Monte Carlo Relaxation).} Under the smoothness, boundedness, coercivity, and stochastic approximation assumptions stated in Section~\ref{appdx:general_conditions}, the iterates \( \theta^{(t)} \) produced by stochastic gradient ascent satisfy
\[
\nabla_\theta \mathbb{E}[\widehat{\mathcal{L}}(\theta^{(t)})] \to 0 \quad \text{as } t \to \infty.
\]
This result follows from stochastic approximation theory~\cite{robbins1951stochastic} combined with the reparameterization-gradient method for continuous relaxations~\cite{kingma2014auto}.

\begin{proof}
Let \( \mathcal{L}(\theta) \) denote the ELBO of the MBT-based Mix- and MoE-DPO model:
\[
\mathcal{L}(\theta) = \mathbb{E}_{(x, y^+, y^-) \sim \mathcal{D}} \left[ \log \sum_{k=1}^K w_k(x) \cdot \frac{\exp(r_k(x, y^+))}{\exp(r_k(x, y^+)) + \exp(r_k(x, y^-))} \right] - \text{KL regularization}.
\]
Due to the latent discrete structure (expert index \( z \)), we introduce a continuous relaxation via the Gumbel--Softmax (Concrete) distribution to enable reparameterization.

Let \( z \sim \mathrm{RelaxedCategorical}(\tau, w(x)) \) be the reparameterized latent variable, where \( \tau \) is the temperature, and let \( f(z; \theta) \) be the differentiable reparameterization of the ELBO integrand. Then the Monte Carlo estimate of the relaxed ELBO is
\[
\widehat{\mathcal{L}}(\theta) = \frac{1}{M} \sum_{m=1}^M f(z^{(m)}; \theta), \quad z^{(m)} \sim \mathrm{RelaxedCategorical}(\tau, w(x)),
\]
which is an unbiased estimator of the relaxed objective:
\[
\mathbb{E}[\widehat{\mathcal{L}}(\theta)] = \mathcal{L}_\tau(\theta),
\]
where \( \mathcal{L}_\tau(\theta) \) is the temperature-smoothed ELBO that converges to the exact ELBO as \( \tau \to 0 \).

\textbf{(i) Unbiased gradient estimate.}
Using the reparameterization trick, we have
\[
\nabla_\theta \mathbb{E}[\widehat{\mathcal{L}}(\theta)] = \mathbb{E}[\nabla_\theta f(z; \theta)],
\]
and \( \nabla_\theta f(z; \theta) \) is an unbiased estimator of \( \nabla_\theta \mathcal{L}_\tau(\theta) \).

\textbf{(ii) Robbins--Monro conditions.}
Assume the learning rates \( \{\eta_t\} \) satisfy \( \sum_t \eta_t = \infty \) and \( \sum_t \eta_t^2 < \infty \).

\textbf{(iii) Variance control.}
Suppose that the variance of the stochastic gradient estimator \( \nabla_\theta \widehat{\mathcal{L}}(\theta^{(t)}) \) is uniformly bounded, and the relaxed objective \( \mathcal{L}_\tau(\theta) \) is continuously differentiable and satisfies the coercivity condition \( \|\theta\| \to \infty \Rightarrow \mathcal{L}_\tau(\theta) \to -\infty \).

Then, by the Robbins--Monro theorem~\cite{robbins1951stochastic}, the stochastic gradient ascent updates
\[
\theta^{(t+1)} = \theta^{(t)} + \eta_t \nabla_\theta \widehat{\mathcal{L}}(\theta^{(t)})
\]
converge to a stationary point of the relaxed objective:
\[
\nabla_\theta \mathbb{E}[\widehat{\mathcal{L}}(\theta^{(t)})] = \nabla_\theta \mathcal{L}_\tau(\theta^{(t)}) \to 0 \quad \text{as } t \to \infty.
\]
\end{proof}

\section{Extensions}

\textbf{User Personalization.} Our variational framework naturally supports personalized alignment by conditioning the gating function on user metadata \( u \in \mathcal{U} \), yielding personalized mixture weights \( w_k(x, u; \phi) \). This extension enables MoE-DPO to tailor expert allocation to individual user profiles without retraining expert policies. In settings where pretrained or fixed experts are available (e.g., domain-specific LLMs), personalization can be implemented by fine-tuning only the gating network, effectively leveraging modular reuse for scalable deployment. Our multi-task review generation experiment (Section~\ref{sec:experiments}) demonstrates this capability: even when experts are fixed, prompt-conditioned gating achieves significant improvements on user-specific reward metrics.

\textbf{Scaling and Sparse Activation.} Our training objective admits closed-form per-expert updates (Theorem~\ref{lemma:Optimal Expert Policy under KL-Regularized Objective}) and modular decomposition of the ELBO, which makes MoE-DPO compatible with sparse expert activation. At inference time, only the top-\(S\) experts (with the highest gating scores) may be activated per input, significantly reducing compute. Empirically, we observe in Section~\ref{sec:experiments} that even with a single routed expert per prompt (oracle routing), alignment quality improves across tasks, suggesting the potential for scalable deployment via sparse MoE architectures. Extensions to dynamic expert growth and token-to-expert allocation can leverage scaling laws developed in~\cite{ludziejewski2024moe}.

\textbf{Differentiable Training via Monte Carlo Relaxation.} While our method employs a classical variational EM algorithm, we recommend Monte Carlo relaxation for large-scale or continuous-group settings. Discrete expert assignments can be approximated using the Gumbel--Softmax reparameterization~\cite{jang2017categorical,kingma2014auto}, enabling differentiable training without explicit E-steps. For continuous groups, such as Lie group mixtures, exponential-map reparameterization~\cite{falorsi2018reparameterizing} supports structured sampling and backpropagation. These techniques facilitate scalable training of MoE-DPO models with integrated routing and multimodal expert structures.

\section{Additional Experiment Details}

We include additional experimental results on Case 2, which employs independent copies of GPT-2 models fine-tuned for particular reward signals. For training, we run Mix-DPO for 3 epochs with a learning rate of \(10^{-5}\), batch size of 64, and the AdamW optimizer. Average test metric values with standard errors are reported in the table below for average- or model-specific inference styles. We note that average results for independent policies outperform the parameter-sharing policy specification of Case 1 and offer comparable results in the test metrics for individual models.

\subsection{Mix-DPO for Multi-Reward Movie Review Generation}

\begin{table}[!htp]
\centering
\small
\caption{Reward scores (mean \( \pm \) SE) for Mix-DPO on the IMDb test set.}
\label{tab:exp1_rewards}
\begin{tabular}{lccc}
\toprule
Model & Sentiment & Informativeness & Grammar \\[0.1cm]
\midrule
Baseline DPO &  0.610 \( \pm \) 0.004 & 0.363 \( \pm \) 0.008 & 0.216 \( \pm \) 0.001 \\[0.1cm]
\toprule
Case 1 (Avg.\ Heads) & \textbf{0.654} \( \pm \) 0.004 & 0.326 \( \pm \) 0.007 & \textbf{0.241} \( \pm \) 0.001\\
Case 1 (Head-Specific)  & & \\
 Head 0 & 0.616  \( \pm \) 0.02 & \textbf{0.396} \( \pm \) 0.007 & \textbf{0.263} \( \pm \) 0.001 \\[0.1cm]
 Head 1 & \textbf{0.720} \( \pm \) 0.003 & \textbf{0.394} \( \pm \) 0.008 & 0.213 \( \pm \) 0.001 \\[0.1cm]
 Head 2 & 0.632 \( \pm \) 0.004 &  0.342 \( \pm \) 0.007 & \textbf{0.267} \( \pm \) 0.001 \\[0.1cm]
Fixed Weights (\(w_i = 1/3\)) &  0.646  \( \pm \) 0.004 & 0.318 \( \pm \) 0.009  & 0.239 \( \pm \) 0.002\\
\toprule
Case 2 (Avg.\ Heads)  & \textbf{0.664} \( \pm \) 0.004 & \textbf{0.350} \( \pm \) 0.006 & 0.232 \( \pm \) 0.002\\
Case 2 (Model-Specific)  &             &             &             \\
 Model 0  & \textbf{0.747} \( \pm \) 0.004 & 0.324  \( \pm \) 0.007 & 0.232 \( \pm \) 0.002  \\ [0.1cm]
 Model 1  & 0.686 \( \pm \) 0.005 & \textbf{0.390} \( \pm \) 0.003 & 0.198 \( \pm \) 0.001\\  [0.1cm]
 Model 2 & 0.638 \( \pm \) 0.004 & 0.371 \( \pm \) 0.006 & \textbf{0.243} \( \pm \) 0.002 \\[0.1cm]
\bottomrule
\end{tabular}
\end{table}

\end{document}